\documentclass{vldb}			%VLDB
\usepackage{balance}  % for  \balance command ON LAST PAGE  (only there!)		%VLDB

\usepackage{balance}  % for  \balance command ON LAST PAGE  (only there!)		%VLDB
\usepackage{times}
\usepackage{amsfonts}
\usepackage{textcomp}
\usepackage{fixltx2e}
\usepackage{amsmath}
\usepackage{amssymb}
\usepackage{latexsym}
\usepackage{algorithm}
\usepackage{algorithmic}
\usepackage{multirow}
\usepackage{float}
\usepackage{color}
\usepackage{helvet}
\usepackage{courier}
\usepackage{graphicx}
\usepackage{subfig}
\usepackage{wrapfig}
\usepackage{framed} % or, "mdframed"
\usepackage[framed]{ntheorem}
\newframedtheorem{frm-thm}{Theorem}

\DeclareMathOperator*{\argmax}{arg\,max}

\makeatletter
\newcommand{\@BIBLABEL}{\@emptybiblabel}
\newcommand{\@emptybiblabel}[1]{}
\makeatother
\usepackage[hidelinks]{hyperref}
\usepackage{url}
\newcommand{\refalg}[1]{Algorithm~\ref{#1}}
\newcommand{\refeqn}[1]{Equation~\eqref{#1}}
\newcommand{\reffig}[1]{Figure~\ref{#1}}
\newcommand{\reftbl}[1]{Table~\ref{#1}}
\newcommand{\refsec}[1]{Section~\ref{#1}}
\newcommand{\method}[1]{\mbox{\textsc{#1}}}

\newcommand{\reals}{\mathbb{R}}

\newcommand{\budget}{\mathbb{B}}
\newcommand{\set}[1]{\mathcal{#1}}
\newcommand{\allhits}{\set{H}}
\newcommand{\hitset}[1]{\set{#1}}
\newcommand{\cconstraint}{\set{C}}
\newcommand{\inferset}{\set{I}}

\newcommand{\dom}[1]{\mathrm{Dom}(#1)}

\newcommand{\reminder}[1]{}
\newcommand{\rmd}[1]{}

\newcommand{\alter}[1]{}
\newcommand{\psl}{PSL}

\newcommand{\tasks}{BETs}
\newcommand{\task}{BET}

\newcommand{\overallmetric}{\Delta_{overall}}
\newcommand{\predicatemetric}{\Delta_{predicate}}

\newcommand{\system}{{KGEval}}
\newcommand{\systemfull}{KG-Evaluation}

\newcommand{\cgraphfull}{Evaluation Coupling Graph}
\newcommand{\cgraph}{ECG}

\newtheorem{theorem}{Theorem}

\title{\system{}: Estimating Accuracy of Automatically Constructed Knowledge Graphs }
\begin{document}

%\author{
%Prakhar Ojha  \\
%Indian Institute of Science \\
%{\small \texttt{prakhar.ojha@csa.iisc.ernet.in}}
%%
%\And
%%
%Partha Talukdar \\ 
%Indian Institute of Science \\ 
%{\small \texttt{ppt@cds.iisc.ac.in}}
%}

%\title{Relational Crowdsourcing and its Application in Knowledge Graph Evaluation}

\numberofauthors{2} %  in this sample file, there are a *total*

\author{
% 1st. author
\alignauthor
Prakhar Ojha\\
       \affaddr{Indian Institute of Science}\\
%       \affaddr{1932 Wallamaloo Lane}\\
%       \affaddr{Wallamaloo, New Zealand}\\
       \email{prakhar.ojha@csa.iisc.ernet.in}
% 2nd. author
\alignauthor
Partha P Talukdar\\
       \affaddr{Indian Institute of Science}\\
%       \affaddr{P.O. Box 1212}\\
%       \affaddr{Dublin, Ohio 43017-6221}\\
       \email{ppt@serc.iisc.in}
}

\maketitle

\begin{abstract}

Automatic construction of large knowledge graphs (KG) by mining web-scale text datasets has received considerable attention recently.
%, resulting in KGs like Google Knowledge Vault, NELL, etc. 
Estimating accuracy of such automatically constructed KGs is a challenging problem due to their size and diversity. 
This important problem has largely been ignored in prior research -- we fill this gap and propose \system{}. 
%
%To the best of our knowledge, this is the first such system of its kind. 
%\system{} is an instance of a  novel crowdsourcing paradigm where dependencies among tasks posted to the crowd workers are exploited. 
%We introduce a novel crowdsourcing paradigm where dependencies among tasks posted to the crowd-workers are exploited.% to evaluate KGs at predicate-relation granularity.
%
\system{} binds facts of a KG using \textit{coupling constraints} and crowdsources the facts that \textit{infer} correctness of large parts of the KG.
We demonstrate that the objective optimized by \system{} is submodular and NP-hard, allowing guarantees for our approximation algorithm.
Through extensive experiments on real-world datasets, we demonstrate that \system{} is able to estimate KG accuracy more accurately compared to other competitive baselines, while requiring significantly lesser number of human evaluations.  

\end{abstract}

\section{Introduction}
\label{sec:intro}

Automatic construction of Knowledge Graphs (KGs) from Web documents has received significant interest over the last few years, resulting in the development of several large KGs consisting of hundreds of predicates (e.g., \textit{isCity}, \textit{stadiumLocatedInCity(Stadium, City)}) and millions of instances of such predicates called beliefs (e.g., \textit{(Joe Luis Arena, stadiumLocatedInCity, Detroit)}). Examples of such KGs include NELL \cite{NELL-aaai15}, Knowledge-Vault \cite{dong2014knowledge} etc. 

%%
%%\alter{Other possible focused start - There are tons of work on creating large scale KG but very few on how to evaluate them.}
%Such KGs contain hundreds of `predicate-relations' (e.g., \textit{isCity}, \textit{stadiumLocatedInCity}) and millions of their instances called `beliefs' (e.g., \textit{(Joe Luis Arena, stadiumLocatedInCity, Detroit)}). 
%%Due to imperfections in automatic models and unreliability of source web documents, many of these beliefs in the KG may be noisy.
%Many incorrect beliefs are also found in KGs due to imperfect \emph{automatic-construction} models and unreliable source web documents. 
%%Estimating accuracy of such KGs is important 

Due to imperfections in the automatic KG construction process, many incorrect beliefs are also found in these KGs. 
Overall accuracy of a KG can quantify the  effectiveness of its construction-process.
Having knowledge of accuracy for each predicate in the KG can highlight its strengths and weaknesses and provide targeted feedback for improvement.
Knowing accuracy at such predicate-level granularity is immensely helpful for Question-Answering (QA) systems that integrate opinions from multiple KGs \cite{samadi2015askworld}.
For real-time QA-systems, being aware that a particular KG is more accurate than others in a certain domain, say sports, helps in restricting the search over relevant and accurate subsets of KGs, thereby improving QA-precision and response time.
In comparison to the large body of recent work focused on construction of KGs, the important problem of accuracy estimation of such large KGs is unexplored -- we address this gap in this paper.

True accuracy of a KG (or a predicate in the KG) may be estimated by aggregating human judgments on correctness of each and every belief in the KG (or predicate)\footnote{Please note that the belief evaluation process can't be completely automated. If an algorithm could accurately predict correctness of a belief, then the algorithm may as well be used during KG construction rather than during evaluation.}. Even though crowdsourcing marketplaces such as Amazon Mechanical Turk (AMT)\footnote{ \scriptsize AMT:
 \url{http://www.mturk.com}
} provide a convenient way to collect human judgments, accumulating such judgments at the scale of larges KGs  is prohibitively expensive. We shall refer to the task of manually classifying a single belief as true or false as a Belief Evaluation Task (\task{}). 
%
%Human judgments are necessary in estimating KG accuracy and the process cannot be completely automated.
%If a software could effectively predict a belief's validity, it may as well be utilized during KG construction rather than evaluation. 
%Crowdsourcing marketplaces such as Amazon Mechanical Turk (AMT)\footnote{ \scriptsize AMT:
% \url{http://www.mturk.com}
%} have emerged as a convenient way to collect human judgments for Human Intelligence Tasks (HITs). 
%In our context, a HIT corresponds to a Belief Evaluation Task (\task{}): classify a single belief as true or false.
%Ideally, we would like to determine the validity of every belief in the KG to get its true accuracy, however this is prohibitively expensive for large KGs.
%%A \task{} in this context is determining the correctness of a single belief. 
%%However, even with spending $\$0.01$ per belief evaluation, and 1 worker per belief, we will require $\$10,000$ to evaluate a KG with 1 million beliefs. 
%%This is , as available budgets are usually orders of magnitude smaller and most KGs are larger.% than 1 million. 
%
Thus, the crucial problem is:
	{%\center 
		%\textit{How can we select a subset of beliefs (\tasks{}) to evaluate which will be (1) the best estimate of true (unknown) accuracy of the KG? and (2) accommodated within limited available budget. }
		\textit{How can we select a subset of beliefs  to evaluate which will give the best estimate of true (but unknown) accuracy of the overall KG and various predicates in it?} 
	}
	
A naive and popular approach is to evaluate randomly sampled subset of beliefs from the KG. 
Since random sampling ignores relational-couplings present among the beliefs, it usually results in oversampling and poor accuracy estimates. 
Let us motivate this through an example.
%We will show using \system{}, that the benefits of harnessing relational-information among beliefs to get better accuracy estimates with fewer samples as compared to random sampling.
	
\begin{figure}[t]
\begin{center}
	\includegraphics[width=0.48\textwidth]{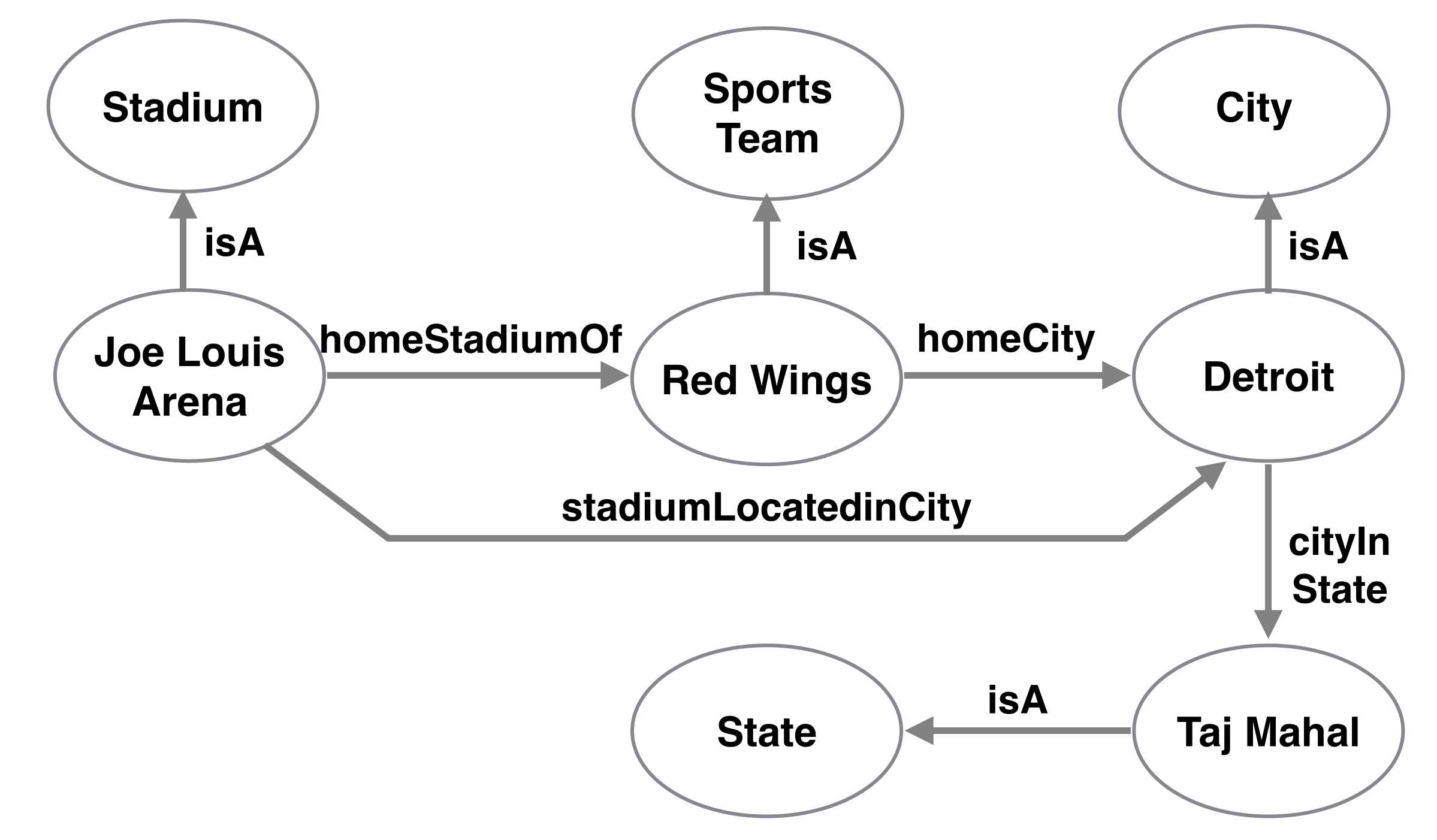}
	\caption{\label{fig:motivating_exmple}{\small Knowledge Graph (KG) fragment used as motivating example in the paper. \system{}, the method proposed in this paper, is able to estimate the true accuracy 75\% of this KG by evaluating only three beliefs (out of 8). In contrast, random evaluation, a popular alternative, gives an estimate of 66.7\% after evaluating same number of beliefs. Please see \refsec{sec:intro} for details.
	%In this figure, if we evaluate that \textit{(Joe Louis Arena - homeStadiumOf - Red Wings)} is correct, then we can automatically infer, using ontology information, that \textit{(Red Wings - isA - Sports Team)} and \textit{(Joe Louis Arena - isA - stadium)} are correct too.
	}}
\end{center}
\end{figure}

\begin{figure*}[t]
	\begin{center}
	\includegraphics[width=0.99\textwidth]{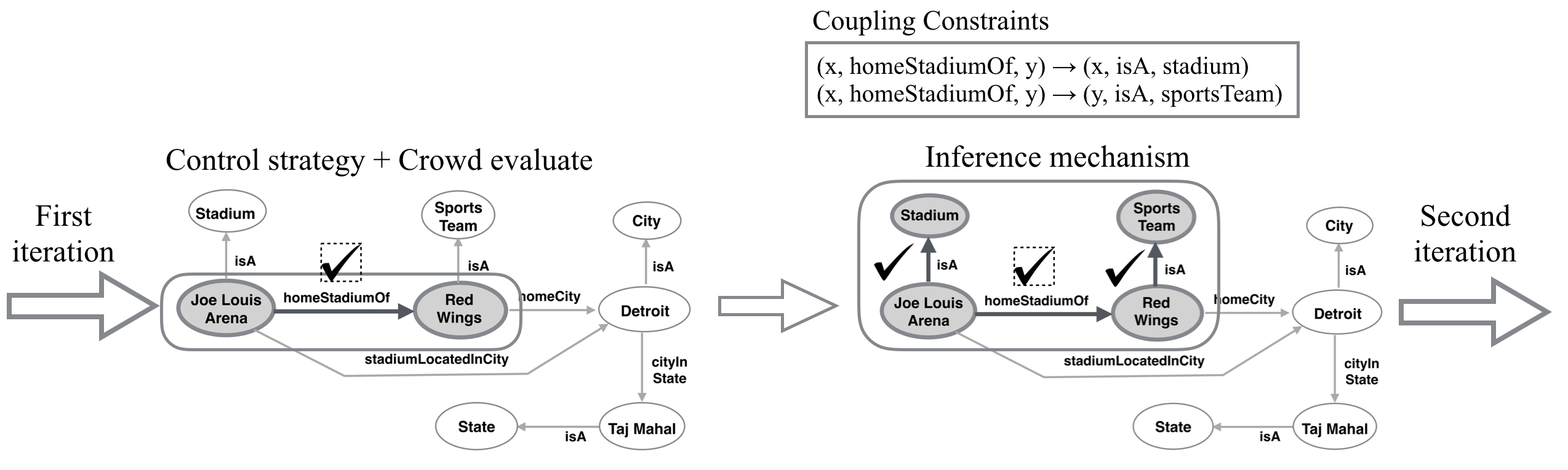}
		%\subfloat[]{\includegraphics[width=0.99\textwidth]{figures/pipeline/pipe.png}\label{fig:cntl1}}
	%	\subfloat[]{\includegraphics[width=0.33\textwidth]{figures/pipeline/inference1.png}\label{fig:inf1}}
%		\subfloat[]{\includegraphics[width=0.33\textwidth]{figures/pipeline/control2.png}\label{fig:cntl2}}
	%	\subfloat[]{\includegraphics[width=0.26\textwidth]{figures/pipeline/inference2.png}\label{fig:inf2}}
		\caption{\label{fig:pipeline}\small Demonstration of one iteration of \system{}.
		Control mechanism selects a belief whose correctness is evaluated from crowd. In the above example,  \textit{(J.L. Arena, homeStadiumOf, Red Wings)} is crowd-evaluated to be true (indicated by tick with dotted square). 
		The inference mechanism take this evaluation along with coupling constraints among beliefs (\refsec{subsec:ECG}) to \textit{infer} correctness of other beliefs (shown by tick without square). All such crowdsourced and inferred evaluations are aggregated to get accuracy estimate. This iterative process continues until convergence.
		(\refsec{subsec:overview} and \refsec{sec:model}).
		} 
	\end{center}
\end{figure*}

\subsubsection*{Motivating example:} 
We motivate efficient accuracy estimation through the KG fragment shown in \reffig{fig:motivating_exmple}. 
Here, each belief of the form (\textit{RedWings, isA, SportsTeam}) is an edge in the graph.
There are six correct and two incorrect beliefs (the two incident on \textit{Taj Mahal}), resulting in an overall accuracy of \mbox{$75\%(=6/8)$} which we would like to estimate. Additionally, we would also like to estimate accuracies of the predicates:  \textit{homeStadiumOf}, \textit{homeCity}, \textit{stadiumLocatedInCity}, \textit{cityInState} and \textit{isA}. 

%%%%%%%% 
%\rmd{Put this in proper format -- PPT: do we need this? \\
%(J.L Arena, homeStadiumOf, Red Wings), True \\
%(J.L Arena, isA, Stadium) , True \\
%(Red Wings, homeCity, Detroit) , True \\
%(Red Wings, isA, Sports Team) , True \\
%(Detroit, cityInState, Taj Mahal), False \\
%(Detroit, isA City) , True \\
%(Taj Mahal, isA, State), False}
%
%Note that the sample KG does not violate any type coupling constraint.
%However, it has wrongly learned the fact that (Detroit, cityInState, Taj Mahal) and (Taj Mahal, isA, State). \rmd{rectify language}.
%%%%%%%%%

%We would like to estimate this true accuracy by evaluating as few beliefs as possible. 
	
We now demonstrate how \textit{coupling constraints} among beliefs may be exploited for faster and more accurate accuracy estimation. %We can \textit{couple} the beliefs in this example by drawing constraints among them. 
\emph{Type consistency} is one such coupling constraint.
For instance, we may know from KG ontology that the \textit{homeStadiumOf} predicate connects an entity from the \textit{Stadium} category to another entity in the \textit{Sports Team} category. 
Now, if \textit{(Joe Louis Arena, homeStadiumOf, Red Wings)} is sampled and is evaluated to be correct, then from these \textit{type constraints} we can infer that \textit{(Joe Louis Arena, isA, Stadium)} and \textit{(Red Wings, isA, Sports Team)} are also correct.  
Similarly, by evaluating \textit{(Taj Mahal, isA, State)} as false, we can infer that \textit{(Detroit, cityInState, TajMahal)} is incorrect. Please note that even though type information may have been used during KG construction and thereby the KG itself is type-consistent, it may still contain incorrect type-related beliefs. 
An example is the \textit{(Taj Mahal, isA, State)} belief in the type-consistent KG in \reffig{fig:motivating_exmple}.
%Please note that these inferences are directional.
%In addition to \textit{type} coupling constraints, \textit{Horn clauses} \cite{NELL-aaai15,lao2011random}, such as (\textit{homeStadiumOf $\land$ playsInCity $\rightarrow$ stadiumLocatedInCity}), can also be used. 
%
Additionally, we have \textit{Horn-clause coupling constraints} \cite{NELL-aaai15,lao2011random}, such as \textit{homeStadiumOf(x, y) $\land$ homeCity(y, z) $\rightarrow$ stadiumLocatedInCity(x, z)}.
By evaluating \textit{(Red Wings, homeCity, Detroit)} and applying this horn-clause to the already evaluated facts mentioned above, we infer that \textit{(Joe Louis Arena, stadiumLocatedInCity, Detroit)} is also correct.
We explore generalized forms of  these constraints in \refsec{subsec:ECG}.

%Thus, by exploiting the coupling constraints among beliefs and by evaluating only three of them, we are able to exactly estimate true accuracy of $75\%. $
Thus, evaluating only three beliefs, and exploiting constraints among them, we exactly estimate the \textit{overall} true accuracy as $75\%$ and also cover all predicates.
In contrast, the empirical accuracy by randomly evaluating three beliefs, averaged over 5 trials, is  $66.7\%$.  

\subsubsection*{Our contributions}
We make the following contributions in this paper:

%To estimate accuracy of a KG, our strategy will be to extract coupling constraints which bind beliefs of the KG and have a mechanism to infer evaluation labels for a larger set of beliefs given labels of fewer beliefs.
%We select those beliefs to validate from crowd workers which infer evaluation labels for larger parts of the KG.

%\reminder{THINK AGAIN}
\begin{itemize}
%	\item We introduce \systemfull{} (\system{}), a new paradigm aimed at crowdsourcing over multi-relational data where the number of potential \tasks{} far exceeds what can be accommodated within limited available budget.
	\item Initiate a systematic study into the unexplored yet important problem of evaluation of automatically constructed Knowledge Graphs.     
	\item We present \system{}, a novel crowdsourcing-based system for estimating accuracy of large knowledge graphs (KGs). \system{} exploits dependencies among beliefs for more accurate and faster KG accuracy estimation.	
	
	\item Demonstrate \system{}'s effectiveness through extensive experiments on real-world KGs, viz.,  NELL \cite{NELL-aaai15} and Yago \cite{suchanek2007yago}. We also evaluate \system{}'s robustness and scalability. 
%%	\item We apply \system{} for estimating quality of automatically constructed KGs
%%	such as NELL \cite{NELL-aaai15} and Yago \cite{suchanek2007yago}. 
%%	We calculate their accuracy at predicate level and overall granularity. 
%%	To our knowledge, this is the first such system of its kind for multi-relational knowledge graphs.
%	%
%%	\item We demonstrate that the objective optimized by \system{} using our inference method is in fact submodular, and hence allowing for application of efficient approximation algorithms with guarantees. 
%%
%	\item We evaluate the effectiveness, robustness and scalability of \system{} by extensive experiments. 
%	%The data and code used in this paper publicly available.
%	We will make the data and code publicly available upon publication.
\end{itemize}  
%%We next formulate our problem in \refsec{sec:problemform} and in doing so also establish  the notations used throughout this paper. 
%%We then discuss the applicability of our approach for KG evaluation in \refsec{sec:model} and present experimental results in \refsec{sec:expmnt}.
%\rmd{We haven't said about RelCrowd as our contribution.}
All the data and code used in the paper will be made publicly available upon publication.

%
%\alterE{}
%\input{sections/related}
%\section{Relational Crowdsourcing}

%\section{Problem Formulation}
\section{Overview and Problem Statement}
\label{sec:probform}

%\begin{figure*}[t]
%	\begin{center}
%	\includegraphics[width=0.99\textwidth]{figures/pipeline/pipe.png}
%		%\subfloat[]{\includegraphics[width=0.99\textwidth]{figures/pipeline/pipe.png}\label{fig:cntl1}}
%	%	\subfloat[]{\includegraphics[width=0.33\textwidth]{figures/pipeline/inference1.png}\label{fig:inf1}}
%%		\subfloat[]{\includegraphics[width=0.33\textwidth]{figures/pipeline/control2.png}\label{fig:cntl2}}
%	%	\subfloat[]{\includegraphics[width=0.26\textwidth]{figures/pipeline/inference2.png}\label{fig:inf2}}
%		\caption{\label{fig:pipeline}\small Demonstration of one iteration of \system{}.
%		Control mechanism selects a belief to crowdsource and  inference mechanism uses this evaluated belief and coupling constraints to evaluate more beliefs.
%		(\refsec{subsec:overview}).
%		%
%		%
%		\rmd{PPT: show evaluations}
%		} 
%	\end{center}
%\end{figure*}

\subsection{\system{}: Overview}
\label{subsec:overview}

The core idea behind \system{} is to estimate correctness of as many beliefs as possible while evaluating only a subset of them using humans through crowdsourcing. %\system{} achieves this goal by using coupling constraints in an inference procedure  to propagate evaluations from already evaluated beliefs to currently unevaluated beliefs. 
\system{} achieves this using an iterative algorithm which alternates between the following two stages:
		\begin{itemize}
			\item {\bf Control Mechanism} (\refsec{sec:control}): In this step, \system{} selects the belief which is to be evaluated next using crowdsourcing.
			\item {\bf Inference Mechanism} (\refsec{subsec:eval_prop}): Here \system{} uses coupling constraints and beliefs  evaluated so far to automatically estimate  correctness of additional beliefs.
		\end{itemize}

\noindent
This iterative process is repeated until there are no more beliefs to be evaluated, or until a pre-determined number of iterations are processed. 
One iteration of  \system{} over the KG fragment from \reffig{fig:motivating_exmple} is shown in \reffig{fig:pipeline}. 
Firstly, \system{}'s control mechanism selects the belief \textit{(John Louis Arena, homeStadiumOf, Red Wings)}, whose evaluation is subsequently crowdsourced. 
Next, the inference mechanism uses the so evaluated belief and type coupling constraints (not shown in figure) to infer that \textit{(John Louis Arena, isA, Stadium)} and \textit{(Red Wings, isA, Sports Team)} are also true. In the next iteration, the control mechanism selects another belief for evaluation, and the process continues.

%In this section we outline the \system{} model and establish notations used in the rest of this paper.
% 
%Each component of the model is detailed in next \refsec{sec:model}.
%For a given knowledge graph, we first extract coupling constraints (\refsec{subsec:ECG}) and construct Evaluation Coupling Graph (ECG) (\refsec{sec:ecg_construct}).
%With a few evaluated \tasks{}, we use an \textit{Inference Mechanism} to propagate their evaluation labels to other non-evaluated beliefs.
%We finally address the question of choosing next \task{}(s) to be evaluated using crowd, which constitutes the \textit{Control Mechanism}. 
%Notations below formalize coupling constraints and coupling graph. 
%The need for such explicit modeling will become apparent when we develop subsequent theory.

\begin{table}[t]
\begin{small}
\begin{center}

%\begin{tabular}{ |p{1.8cm}|p{2.7cm}|p{1.8cm}|  }
% \begin{tabular}{|lp{4.2cm}|}
  \begin{tabular}{|p{2.7cm}|p{4.5cm}|}
 \hline
 Symbol & Description  \\

 \hline

 $\allhits{} = \{h_1, \ldots, h_n\}$ & Set of all $n$ Belief Evaluation Tasks (\tasks{})\\
 
% $\budget{}$ &  Total allocated budget\\
 
 $c(h) \in \reals_{+}$ &  Cost of labeling $h$ from crowd\\

%  $h_i \in \allhits{}$ & Single binary classification HIT from $\allhits{}$\\

% , generally can also act on sets of HITs to return binary vector. $l(\allhits{}) = \{0,1\}^n$ 
 
%  $\cconstraint{}_i : \allhits{}_i  \rightarrow \reals$ & Coupling constraint to enforce consistency among $l(\allhits{}_i)$\\
 
 $\cconstraint{} = \{(\cconstraint{}_i, \theta_i)\}$ &  Set of coupling constraints $\cconstraint{}_i$ with weights $\theta_i \in \reals_{+}$\\
   
 $t(h) \in \{0, 1\}$ &  True label of $h$\\  
  
 $l(h) \in \{0, 1\}$ &  Estimated evaluation label of $h$\\ 
%$\theta_i \in \reals$ & Strength of enforcing constraint $\cconstraint{}_i $ relative to other $\cconstraint{}_j \in \cconstraint, i\neq j$\\

$\allhits{}_i = \dom{\cconstraint_i{}} $ & $\allhits{}_i \subseteq \allhits{}$ which participate in $\cconstraint{}_i$ \\

$G = (\allhits{} \cup \cconstraint{}, \set{E})$ &  Bipartite Evaluation Coupling Graph. $e\in\set{E}$ between $\allhits{}_j$ and $\cconstraint{}_j$ denotes $\allhits{}_j \in \dom{\cconstraint_j{}}$  .\\

 $\set{Q} \subseteq \allhits{}$ &  \tasks{} evaluated using crowd \\
 
 $ \inferset\big(G,\set{Q} \big) \subseteq \allhits{}$ & Inferable Set for evidence $\set{Q}$: \tasks{} for which labels can be inferred by inference algorithm \\
  
 \mbox{$\Phi(\set{Q})=$}  $~~~~~~~~\frac{1}{|\set{Q}|} \sum_{h \in \set{Q}} t(h)$ & True accuracy of evaluated \tasks{} $\set{Q}$ \\

 \hline

\end{tabular}
 \caption{\small \label{tbl:notations} Summary of notations used (\refsec{subsec:notation}).} 
 %in this paper. Refer  
 %for more details.}
\end{center}
\end{small}	

\end{table}
%
%
%{\bf Notations}: We are given a KG with $n$ beliefs, evaluation of each one of which forms a binary categorization-type Human Intelligence Task (\task{}), i.e., we have the \task{} set $\allhits{} = \{h_1, \ldots, h_n\}$. Each \task{} incurs crowdsourcing cost $c(h_i) \in \reals_{+}$, and we are given a total budget $\budget{}$. % \ll \sum_{h \in \allhits{}} c(h)$. 

\subsection{Notations and Problem Statement}
\label{subsec:notation}

%We are given a KG with $n$ beliefs. %to evaluate as \textit{True} or \textit{False}. 
%Evaluating a single belief as true or false forms a Belief Evaluation Task (\task{}), resulting in set $\allhits{} = \{h_1, \ldots, h_n\}$. 
%%We are given a total budget $\budget{}$ and 
%Each \task{} incurs crowdsourcing cost $c(h_i) \in \reals_{+}$, uniform in our case. % \ll \sum_{h \in \allhits{}} c(h)$. 
%%Mapping function  $l(h_i) \in \{0, 1\}$  returns the value of evaluation label of $h_i$. Label for $h_i$ could either be gold expert label $l_g(h_i)$ or its estimate $l_u(h_i)$ as returned by noisy crowd worker $u$.
%Mapping function  $t(h_i) \in \{0, 1\}$  returns the true evaluation label of $h_i$ where $\{0, 1\}$ symbolize False and True respectively. 
%Function $l(h_i) \in \{0, 1\}$ is the estimate of $t(h_i)$ after aggregation over noisy crowd workers.

We are given a KG with $n$ beliefs. 
Evaluating a single belief as true or false forms a Belief Evaluation Task (\task{}). 
Coupling constraints are derived by determining relationships among \tasks{}, which we further discuss in \refsec{subsec:ECG}. 
Notations used in the rest of the paper are summarized in \reftbl{tbl:notations}.

\textit{Inference algorithm} helps us work out evaluation labels of other \tasks{} using constraints $\cconstraint{}$. 
For a set of already evaluated \tasks{} $\set{Q} \subseteq \allhits{}$, we define \textit{inferable set} 
%$ \inferset\big(G,\set{Q}, \Theta \big) \subseteq \allhits{}$ to be the set of \tasks{} whose evaluation 
$ \inferset\big(G,\set{Q} \big) \subseteq \allhits{}$ as \tasks{} whose evaluation 
labels can be deduced by the inference algorithm.
%$\Theta$ denotes the parameters required by inference algorithm.
We calculate the average true accuracy of a given set of evaluated \tasks{} $\set{Q} \subseteq \inferset\big(G,\set{Q} \big) \subseteq \allhits$ by 
%\begin{center}
$\Phi(\set{Q}) = \frac{1}{|\set{Q}|} \sum_{h \in \set{Q}} t(h)$.
%\end{center}

%\noindent{\bf Problem Statement}: 
%\subsection{Problem Statement}
%\label{subsec:probStatement}

%\rmd{PPT: hide following? How to sample a subset of \tasks{} $\hitset{Q} \subseteq \allhits$ to crowdsource, such that the resulting inferable set $ \inferset\big(G,\set{Q} \big)$ provides best estimate of $\Phi(\allhits{})$ under all possible choices of $\hitset{Q}$?
%%, while remaining within budget $\budget$? 
%If the aggregated labels $l(h)$ are correlated with true labels $t(h) ~\forall h\in \allhits$, 
%%(i.e., the inference algorithm expands binary scores $\{0,1\}$ of other \tasks{} given seed set $\hitset{Q}$ well), 
%then maximizing size of inferable set minimizes the uncertainty in $\Phi$.
%In other words,}
\system{} aims to sample and crowdsource a \task{} set $\hitset{Q}$ with the largest inferable set, 
% subject to budget constraint. 
and solves the optimization problem:
	\begin{eqnarray}
		\arg\max_{\hitset{Q} \subseteq \allhits}~\left|\inferset\big(G,\set{Q} \big)\right|
%		,~~\mathrm{s.t.}~~ \sum_{h \in \hitset{Q}} c(h) \le \budget 
		\label{eqn:rc_obj}
	\end{eqnarray} 
\noindent

\section{\system{}: Method Details}
\label{sec:model}

	\begin{figure}[t]
	\begin{center}		\includegraphics[width=0.4\textwidth]{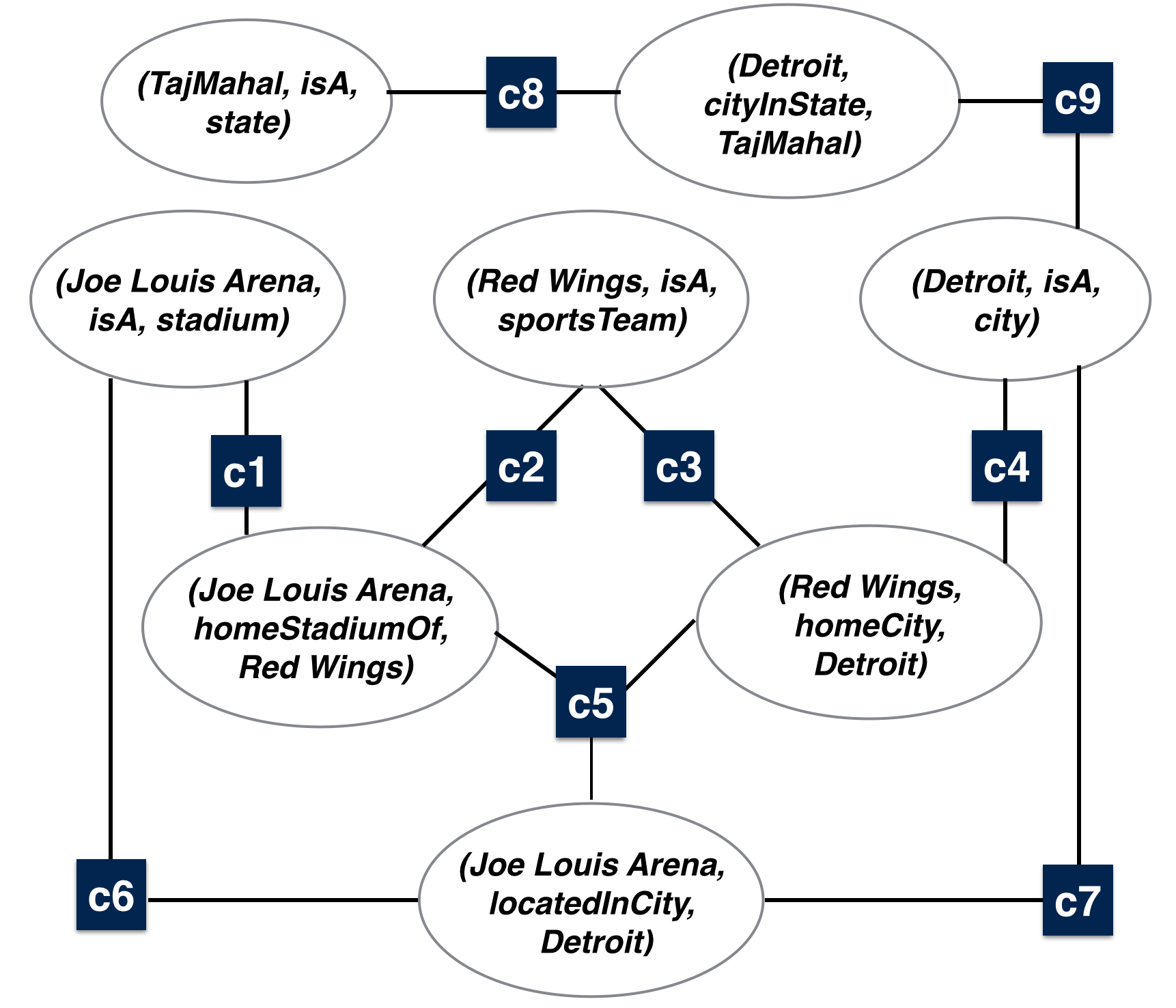}
		%\includegraphics[width=0.52\textwidth]{figures/sports_example_factor_graph_2.png}
%OLD_complex_2_ecg%		
%	\caption{\label{fig:motivating_exmple_fg}\cgraphfull{} (\cgraph{}) constructed for example in  \reffig{fig:motivating_exmple}. Each edge (\task{}) there is represented as a node in the current figure, with two or more \task{}-nodes connected by coupling constraint node. Each coupling constraint node enforces a  constraint among the correctness scores of its adjacent \task{} nodes. For example, the \textit{c7} and \textit{c5} factors enforce the domain and range type coupling constraints with respect to the \textit{(Joe Louis Arena - homeStadiumOf - Red Wings)} \task{}. (see \refsec{sec:ecg_construct})}
	
	\caption{\label{fig:motivating_exmple_fg}{\small \cgraphfull{} (\cgraph{}) constructed for example in  \reffig{fig:motivating_exmple}. 
Each edge (\task{}) in  \reffig{fig:motivating_exmple} is represented as a node here.
%Each coupling constraint node enforces a  constraint among the evaluation labels of its adjacent \task{} nodes.	
%Edges represent participation of \task{}-nodes with particular coupling constraint.
	%For example, the \textit{c1} and \textit{c2} factors enforce the domain and range type coupling constraints with respect to the \textit{(Joe Louis Arena - homeStadiumOf - Red Wings)} \task{}. 
	(\refsec{sec:ecg_construct})}}
	\end{center}
	\end{figure}

%In this section, we describe how KG evaluation may be posed as an instance of  \systemfull{}, where each \task{} tries to categorize a belief (an edge in KG) into correct or incorrect. 
%%\reffig{fig:amt} shows a sample \task{} we posted on AMT.
%Given a set of evaluated \tasks{} $\hitset{Q}$, $\Phi\big(\inferset(G,\set{Q}, \Theta )\big)$ is accuracy of the corresponding beliefs in its inferable set, whereas $\Phi(\allhits{})$ is the overall gold accuracy of entire KG. 
%We would like to estimate this unknown quantity by crowdsourcing a small subset $\hitset{Q}$ identified by optimizing the \system{} objective shown in \refeqn{eqn:rc_obj}.

%{\large {\bf Outline}}: 
%\alterS{}
%In order to solve the \system{} objective in the context of KG evaluation, we proceed as follows:

%\alterE{}

In this section, we describe various components of \system{}. First, we present coupling constraint details in \refsec{subsec:ECG}. 
Instead of working directly over the KG, \system{} combines all KG beliefs and constraints into an \textit{Evaluation Coupling Graph (ECG)}, and works directly with it. 
Construction of ECG is described in \refsec{sec:ecg_construct}. 
Inference and Control mechanisms are then described in \refsec{subsec:eval_prop} and \refsec{sec:control}, respectively. 
The overall algorithm is presented in \refalg{alg:KGEalgo}.

\subsection{Coupling Constraints}
\label{subsec:ECG}

\noindent
%As we saw in our motivating example of \refsec{sec:intro}, \textit{(Joe Louis Arena, homeStadiumOf, Red Wings)} restricts `Red Wings' to be of the \textit{type} `sportsTeam'.
%
Beliefs are coupled in the sense that their truth labels are dependent on each other.
We use any additional coupling information like metadata, KG-ontology etc., to derive coupling constraints that  can help infer evaluation label of a \task{} based on labels of other \task{}(s). 
In this work, we derive constraints $\cconstraint{}$ from the KG ontology and link-prediction algorithms, such as PRA \cite{lao2011random} over NELL and AMIE \cite{galarraga2013amie} over Yago.
These rules are jointly learned over entire KG with millions of facts and we assume them to be true.
The alternative of manually identifying such complex rules would be very expensive and unfeasible.

We use conjunction-form first-order-logic rules and refer to them as \textit{Horn clauses}. 
Examples of a few such coupling constraints are shown below.
\begin{center}
%$\cconstraint_4$: \textit{(Red Wings, playsInCity, Detroit)} $\rightarrow$ \textit{(Detroit, isA, city)} \\
%$\cconstraint_5$: \textit{(Joe Louis Arena, homeStadiumOf, Red Wings)} $\land$ \textit{(Red Wings, playsInCity, Detroit)} $\rightarrow$ \textit{(Joe Louis Arena, stadiumLocatedInCity, Detroit)}
%
%$\cconstraint_{7}$: \textit{homeStadiumOf} $\rightarrow$ \textit{domain(stadium)} \\
%$\cconstraint_{5}$: \textit{homeStadiumOf} $\rightarrow$ \textit{range(sportsTeam)} \\
%$\cconstraint_4$: \textit{athleteHomeStadium} $\land$ \textit{homeStadiumOf} $\land$ \textit{athletePlaysInTeam}$^{-1}$~ $\rightarrow$ \textit{teammates}
%
%%%% Old complex factor graph _ 2
%%%%
%$\cconstraint_{7}$: \textit{homeStadiumOf(S,T)} $\rightarrow$ \textit{(S, isA, stadium)} \\
%$\cconstraint_{5}$: \textit{homeStadiumOf(S,T)} $\rightarrow$ \textit{(T, isA, sportsTeam)} \\
%$\cconstraint_4$: \textit{athleteHomeStadium(A1,S)} $\land$ \textit{homeStadiumOf(S,T)} $\land$ \textit{athletePlaysInTeam(A2,T)}$^{-1}$~ $\rightarrow$ \textit{teammates(A1,A2)}
%%%
%%% FOR NEW SIMPLER ECG
%$\cconstraint_{1}$: \textit{(JoeLouisArena, homeStadiumOf, RedWings)} $\rightarrow$ \textit{(JoeLouisArena, isA, stadium)} \\
%$\cconstraint_{2}$: \textit{(JoeLouisArena, homeStadiumOf, RedWings)} $\rightarrow$ \textit{(RedWings, isA, sportsTeam)} \\
%$\cconstraint_{5}$: \textit{(JoeLouisArena, homeStadiumOf, RedWings)} $\land$ \textit{(RedWings, homeCity, Detroit)} $\to$ \textit{(JoeLouisArena, stadiumLocatedInCity, Detroit)} 

$\cconstraint_{2}$: \textit{(x, homeStadiumOf, y)} $\rightarrow$ \textit{(y, isA, sportsTeam)} \\
$\cconstraint_{5}$: \textit{(x, homeStadiumOf, y)} $\land$ \textit{(y, homeCity, z)} $\to$ \textit{(x, stadiumLocatedInCity, z)} 

\end{center}
%
%{\center $\cconstraint_4$: \textit{(Red Wings, playsInCity, Detroit)} $\rightarrow$ \textit{(Detroit, isA, city)}} 
%{\center $\cconstraint_5$: \textit{(Joe Louis Arena, homeStadiumOf, Red Wings)} $\land$ \textit{(Red Wings, playsInCity, Detroit)} $\rightarrow$ \textit{(Joe Louis Arena, stadiumLocatedInCity, Detroit)}} 
%
%For each $\cconstraint_i{}$, \task{}(s) to the left of the arrow constitute its  \textit{body}, while the \task{} to the right is its \textit{head}. Each $\cconstraint_i{}$ is associated with $\theta_i$ weight, which ranks it among other rules.
Each coupling constraint $\cconstraint{}_i$ operates over $\allhits{}_i \subseteq \allhits{}$ to the left of its arrow and infers label of the \task{} on the right of its arrow.
\tasks{} participating in $\cconstraint{}_i$ are referred to as its domain $\dom{\cconstraint_i{}}$.
%\task{}(s) to the left of the arrow of $\cconstraint_i{}$ constitute its \textit{body} and 
%Horn clauses are frequently occurring patterns in the KG which capture common sense inference knowledge. 
%For example, we know from predefined ontology that \textit{(domain, range)} type signature of \textit{homeCity} relation is \textit{(sportsTeam, city)}. 
%Coupling constraint $\cconstraint_{2}$ enforces type consistency between two \tasks{} sharing the  entity \textit{RedWings}. 
%So the enforcement of this constraint helps us infer that \textit{(RedWings, isA, sportsTeam)} should be correct if \textit{(JoeLouisArena, homeStadiumOf, RedWings)} is evaluated to be correct. 
%
%
$\cconstraint_{2}$ enforces \textit{type consistency} and  $\cconstraint_5$ is an instance of PRA path that conveys if a stadium $S$ is home to a certain team $T$ based in city $C$, is itself located in city $C$.
These constraints have also been successfully employed earlier during knowledge extraction \cite{NELL-aaai15} and integration \cite{pujara2013knowledge}. 
%This constraint helps infer the correctness of a \task{} based on the true-labels of two other \tasks{}. 

Note that these constraints are directional and inference propagates in forward direction. 
For example, inverse of $\cconstraint_2$, i.e.,  correctness of \textit{(RedWings, isA, sportsTeam)} does not give any information about the correctness of \textit{(JoeLouisArena, homeStadiumOf, RedWings)}.

\subsection{\cgraphfull{} (\cgraph{})}
\label{sec:ecg_construct}

\noindent
	Given $\allhits{}$ and $\cconstraint{}$, we construct a graph with two types of nodes: (1) a node for each \task{} $h\in\allhits{}$, and (2) a node for each constraint $\cconstraint{}_i\in\cconstraint{}$. 
%For ease of notation, we shall identify a node by the \task{} or constraint that it corresponds to. 
Each $\cconstraint{}_i$ node is connected to all $h$ nodes that participate in it. 
We call this graph the \textit{\cgraphfull{}} (\textit{\cgraph{}}), represented as $G = (\allhits{} \cup \cconstraint{}, \set{E})$ 
with set of edges 
%$\set{E} = \{(\cconstraint{}_i, h)~|~\forall \cconstraint{}_i \in \cconstraint, \exists h \in \dom{\cconstraint{}_i}\}$. 
$\set{E} = \{(\cconstraint{}_i, h)~| ~h \in \dom{\cconstraint{}_i}~\forall \cconstraint{}_i \in \cconstraint\}$. 
Note that \cgraph{} is a bipartite \textit{factor graph}  \cite{kschischang2001factor} with $h$ corresponding to variable-nodes and $\cconstraint_i{}$ corresponding to factor-nodes. 
	
%	Example of an \cgraph{} constructed out of the KG in \reffig{fig:motivating_exmple} and coupling constraint set $\cconstraint{}$ with $|\cconstraint{}| = 8$ is shown in \reffig{fig:motivating_exmple_fg}. 
\reffig{fig:motivating_exmple_fg} shows \cgraph{} constructed out of the motivating example in  \reffig{fig:motivating_exmple} with $|\cconstraint{}| = 8$ and separate nodes for each of the edges (beliefs or \tasks{}) in KG. 
%In \reffig{fig:motivating_exmple_fg}, there are separate nodes for each of the eight edges (beliefs or \tasks{}) in \reffig{fig:motivating_exmple}. 
%Also, there are eight coupling nodes corresponding to each coupling constraint operating over subset of these \task{} nodes. 
We pose the KG evaluation problem as classification of \task{} nodes in the \cgraph{} by allotting them a label of $1$ or $0$ to represent true or false respectively.

%We want to bring forward the multi-relational property of \systemfull{} into notice. Observe the each first-order constraint node $\cconstraint_j{}$ is assigned $\theta_j$ weight,  which ranks it among other rules.
%Higher weight indicates greater chances of being true and such constraint nodes will be more influential over other lower-weight rules in case of conflicts. 
%We pose the \systemfull{} problem as classification of \task{} nodes in \cgraph{} and covering that essential subset $\set{Q}$ which induces a larger coverage, using coupling constraints, over all tasks.
%%The idea is to select a subset of \tasks{}, $\hitset{Q}$, crowdsource the categorization scores for this subset, and use an inference algorithm to propagate these scores on a subset of the nodes to the remaining \task{} nodes of the \cgraph{} and that its approximation is best we can do

%\reminder{brief about our distinction between Inference and Control mechanism.} We take a modular approach in our crowdsourcing environment to distinguish between inference and control. \textit{Inference} is the part where machine infers additional knowledge based on responses provided by the crowd. \textit{Control} part uses this inference system and decides which task to post to the crowd next. Although we make this explicit distinction, both these parts are inter-connected pieces of entire \system{} system.

%\subsection{Evaluation Propagation using Probabilistic Soft Logic (PSL)}
\subsection{Inference Mechanism}
\label{subsec:eval_prop}
Inference mechanism helps propagate true/false labels of evaluated beliefs to other non-evaluated beliefs using available coupling constraints \cite{bragg2013crowdsourcing}. 
We use Probabilistic Soft Logic (\psl{})\footnote{ \scriptsize
\url{http://www.psl.umiacs.umd.edu}
%http://www.psl.umiacs.umd.edu
} \cite{broecheler:uai10} as our inference engine to implement propagation of evaluation labels. 
%, as discussed in \refsec{subsec:eval_prop}.
\psl{} is a declarative language to reason uncertainty in relational domains via first order logic statements. 
Below we briefly describe the internal workings of \psl{} with respect to the accuracy estimation problem.

%\alterS{}
%The first order logic rules are relaxed to their soft truth values and propagation is done by finding the most likely explanation over evidences. 
\subsubsection*{PSL Background} \psl{} relaxes each conjunction rule $\cconstraint_j{}$ to $\reals{}$ using \textit{Lukaseiwicz t-norms}: $a \land b\to max\{0,a+b-1\}$.
Potential function $\psi_j$ is defined for each $\cconstraint_j{}$ using this norm and it depicts how satisfactorily constraint $\cconstraint_j{}$ is satisfied.
For example, $\cconstraint_5$ mentioned earlier is transformed from first-order logical form to a real valued number by
\begin{equation}
\label{eqn:potential}
\psi_j(\cconstraint_5{}) =  \big(\max \{0, h_x+h_y-1-h_w\}\big)^{2}
\end{equation}
%
%$l_{\cconstraint_5} = h_i \land h_j  =  \max \{0, h_i+h_j-1\}$
%\end{center}
%
%\begin{equation*}
%	\label{eqn:LOGIC}
%	l_{\cconstraint_5} = h_i \land h_j  =  \max \{0, h_i+h_j-1\}
%\end{equation*}
where $\cconstraint_5{} = h_x \land h_y \to h_w $, where $h_x$ denotes the evaluation score $\in [0,1]$ associated with %the label of 
the \task{} \textit{(Joe Louis Arena, homeStadiumOf, Red Wings)}, $h_y$  corresponds to \textit{(Red Wings, homeCity, Detroit)} and $h_w$ to \textit{(Joe Louis Arena, locatedInCity, Detroit)}.
%We can choose $p\in\{1,2\}$ for penalizing constraint violations either linearly or squared (less for small and high for larger violations).
Higher value of potential function $\psi_j$ represents lower fit for constraint $\cconstraint_j{}$. 
%$\psi_j$ denotes the degree of satisfaction of constraint $\cconstraint_j{}$.

%Hence we need to strive for probability distribution which promotes lower values for $\psi_j$.
%\alterE{}

The probability distribution over label assignment is so structured such that labels which satisfy more coupling constraints are more probable. %\cite{broecheler:uai10}.
Probability of any label assignment  
$\Omega\big(\allhits{}\big) \in \{0, 1\}^{|\allhits{}|}$
%$l\big(\inferset(G, \set{Q})\big)$
over \tasks{} in $G$ is given by
\begin{equation}
    \label{eqn:MLE}
	\mathbb{P}\big(\Omega\big(\allhits{}\big)\big) = \frac{1}{Z} \exp 
	\Big[ -\sum\limits_{j=1}^{|\cconstraint{}|} \theta_j \psi_j\big(\cconstraint_j{}\big) \Big]
	%\mathbb{P}\bigg(l\big({\inferset}(G, \set{Q})\big)\bigg) = \frac{1}{Z} \exp 
	%\Big[ -\sum\limits_{j=1}^{|\cconstraint{}|} \theta_j \psi_j\big(\inferset{}(G, \set{Q})\big) \Big] 
\end{equation}
where %$Z= \int_{\hat{\inferset}(\set{Q}, \cconstraint{})} \exp \left[ -\sum_{j=1}^{|\cconstraint{}|} w_j \phi_j(\hat{\inferset}(\set{Q}, \cconstraint{})) \right]$ 
$Z$ is the normalizing constant and $\psi_j$ corresponds to potential function acting over \tasks{} $h\in \dom{\cconstraint_j{}}$. 
%Labels for each belief in ECG $l(G)$ 
%
%Optimization technique searches for the most likely explanation of this joint probability distribution over all \tasks{} $h_i \in {\inferset}(G, \set{Q})$ by \textit{maximum a-posteriori} inference. %\cite{bach:uai13}.
Final assignment of $\Omega(\allhits{})_{PSL} \in \{1,0\}^{|\allhits{}|}$ labels is obtained by solving the \textit{maximum a-posteriori (MAP)}  optimization problem
\begin{center}
\label{eqn:MAP}
%$l_{MAP}\big({\inferset}(G, \set{Q})\big) =  \argmax\limits_{{\inferset}(G, \set{Q})} \mathbb{P}\bigg(l\big({\inferset}(G, \set{Q})\big)\bigg)$
$\Omega\big(\allhits{}\big)_{PSL} =  \argmax\limits_{\Omega(\allhits{})} ~\mathbb{P}\bigg(\Omega\big(\allhits{}\big)\bigg)$
\end{center}
%\label{eqn:MAP}
%{\inferset}(G, \set{Q}, \theta)_{MAP} =  \argmax\limits_{{\inferset}(G, \set{Q}, \theta)} \mathbb{P}({\inferset}(G, \set{Q}, \theta)).
%\end{equation}
%In our experiments, to calculate the accuracy of knowledge graph, it is essential to find the total number of correct beliefs against incorrect ones.
%We threshold these soft label values  to distinguish true from false beliefs i.e., $(l_{MAP}(h) \geq threshold) \to (l_g(h)=l_{MAP}(h))$.

We denote by $M_{PSL}(h, \gamma) \in [0, 1]$ the PSL-estimated score for label $\gamma \in \{1, 0\}$ on \task{} $h$ in the optimization above.

{\bf Inferable Set using PSL}: We define  estimated label for each BET $h$ as shown below.
\begin{small}
\begin{eqnarray*}
	l(h) =
	\begin{cases}
		1~~\mathrm{if}~~M_{PSL}(h, 1) \ge \tau \\
		0~~\mathrm{if}~~M_{PSL}(h, 0) \ge \tau \\
		\varnothing~~\mathrm{otherwise}
	\end{cases}	
\end{eqnarray*}
\end{small}
where threshold $\tau$ is system hyper-parameter. Inferable set is composed of \tasks{} for which inference algorithm (PSL) confidently propagates labels. %This is given by,
\begin{center}
	$\inferset(G, \set{Q}) = \{h \in \allhits{}~|~l(h) \ne \varnothing\}$
% $\inferset(G, \set{Q}) = \{h \in \Omega(\allhits{})_{PSL}: \mathbb{P}(h) \geq \tau \}$
\end{center}
% where threshold $\tau$ is system hyper-parameter. \rmd{PPT: should it be defined using $l\big(\set{Q}\big)_{MAP}$}
%
Note that two \task{} nodes from \cgraph{} can interact with varying strengths through different constraint nodes; this multi-relational structure requires soft probabilistic propagation.
%Hence, given initial evaluation labels of \tasks{} in $\set{Q}$, we can use above PSL-based inference mechanism and propagate their values to a larger set of $\inferset(G, \set{Q})$. 
%by assigning them soft-values $\in [0,1]$.

%%%%%%%%%%%%%%%%
%\subsection{\systemfull{} using PSL is Submodular}

%\subsection{Properties of \system{} and Control Mechanism}
\subsection{Control Mechanism}
\label{sec:control}

%\reminder{PPT: compress}
%\subsubsection*{Submodularity}

Control mechanism selects the \task{} to be crowd-evaluated at every iteration. 
We first present the following two theorems involving \system{}'s optimization in \refeqn {eqn:rc_obj}. 
Please refer to the Appendix for proofs of both theorems. 
\subsubsection*{Submodularity}

\begin{frm-thm}
%\begin{theorem}
\label{thm:submod}
%The evaluation propagation function 
The function optimized by \system{} (\refeqn {eqn:rc_obj}) using the PSL-based inference mechanism is submodular \cite{lovasz1983submodular}.
%\end{theorem}
\end{frm-thm}

Intuitively, the amount of additional utility, in terms of label inference, obtained by adding a \task{} to larger set is lesser than adding it to any smaller subset. 
\noindent
The proof follows from the fact that all pairs of \tasks{} satisfy the regularity condition \cite{jegelka2011submodularity,kolmogorov2004energy}, further used by a proven conjecture \cite{kempe2003maximizing,mossel2007submodularity}.
Refer Appendix \refsec{sec:appendix} for detailed proof.
%Observe all pairs of HITs satisfy regularity condition \cite{jegelka2011submodularity}
%\cite{kolmogorov2004energy} clarify the equivalence of regular energy functions and submodular set functions.
%As any non-negative summation of submodular functions remains submodular \cite{lovasz1983submodular}; and from the structure of \refeqn{eqn:MLE}, we can see that $\sum_{j \in \cconstraint{} } \theta_j \phi_j$ is also submodular.
%We further use conjecture of \cite{kempe2003maximizing}, later proved in \cite{mossel2007submodularity} and draw parallels to our global function of \refeqn{eqn:rc_obj}, making our inference function submodular.

\subsubsection*{NP-Hardness}

\begin{frm-thm}
%\begin{theorem}
\label{thm:nphard}
The problem of selecting optimal solution in \system{}'s optimization (\refeqn {eqn:rc_obj}) 
%initial set $\set{Q}$ in order to maximize inference 
is NP-Hard.
%\end{theorem}
\end{frm-thm}

Proof follows by reducing NP-complete Set-cover Problem (SCP) %is a special case of \system{}-KGE and that it 
to selecting $\set{Q}$ which covers ${\inferset}(G, \set{Q})$.

Given a partially evaluated ECG, the control mechanism aims to select the next set of \tasks{} which should be evaluated by the crowd. However, before going into the details of the control mechanism, we state a few properties involving \system{}'s optimization in \refeqn {eqn:rc_obj}.

%\subsubsection*{Control Mechanism}
%\reminder{TODO}
%\alter{Rev.Req: Are the constraints used only to spread the inference or are they also used to select the question: Yes, constraint rules are for selecting the question also (chose the one with maximum size of inference set) and spread inference itself}
\subsubsection*{Justification for Greedy Strategy} From Theorem \ref{thm:submod} and \ref{thm:nphard}, we observe that the function optimized by \system{} is NP-hard and submodular. 
Results from \cite{nemhauser1978analysis} prove that greedy hill-climbing algorithms solve such maximization problem within an approximation factor of $(1-1/e) \approx$ 63\% of the optimal solution. Hence, we adopt a greedy strategy as our control mechanism. 
We 
%use the PSL-based inference mechanism (and thereby the coupling constraints) to 
iteratively select the next \task{} which gives the greatest increase in size of inferable set.

In this work, we do not integrate all aspects crowdsourcing techniques like aggregating labels, worker's quality estimation etc.
However, we acknowledge their importance and hope to pursue them in our future works.
In Appendix \ref{sec:budget_alloc}, we present a majority-voting based mechanism to handle noisy crowd workers
under limited budget.
We propose a strategy to redundantly post each \task{} to multiple workers and bound their estimation error.
%\reminder{PPT: integrate better}

%More details are presented in next section.
%
	\begin{algorithm}[t]
	%\caption{\label{alg:evaluator}\system{}: \systemfull{} }
   \caption{\label{alg:KGEalgo}\system{}: Accuracy Estimation of Knowledge Graphs}
    \begin{algorithmic}[1]    
      \REQUIRE{$\allhits$: \tasks{}, $\cconstraint$: coupling constraints, $\budget$:  assigned budget, $\set{S}$: seed set, $c(h)$: \task{} cost function,   $\Phi$: \task{} categorization score aggregator}
	  \STATE $G = \method{BuildECG}(\allhits{}, \cconstraint{})$
	  %\STATE $\theta_i = \method{LearnParameters}(\set{S})$
%      \STATE $\hitset{Q} = \set{S}$

	  \STATE $B_r = \budget$
	  \STATE $\hitset{Q}_0 =  \set{S}$, $t = 1$
	  % (\allhits \cup \cconstraint, \set{E})
      \REPEAT  
      	%// Generate candidates \\
%      	\STATE $\hitset{H^{'}} = \method{GetCandidate\tasks{}}(G, \set{Q})$
      	\STATE $h^{*} = \arg\max_{h \in \hitset{H} - Q} \left|\inferset(G, \hitset{Q}_{t-1} \cup \{h\})\right|$
      	\STATE $\method{CrowdEvaluate}(h^{*})$
      	\STATE $\method{RunInference}(\hitset{Q}_{t-1} \cup h^{*})$      	
      	\STATE $\hitset{Q}_t = \inferset(G,\hitset{Q}_{t-1} \cup \{ h^{*}\})$
      	\STATE $B_r = B_r - c(h^{*})$
		\STATE $\hitset{Q} =\hitset{Q} \cup \hitset{Q}_t$     
		%\IF{(All \tasks{} Covered)}
		\IF{$\hitset{Q} \equiv \allhits$}
    		\STATE \method{Exit}
   	    \ENDIF
 	    
 	    \STATE $\mathrm{Acc_t} = \frac{1}{|\set{Q}|} \sum_{h \in \set{Q}} l(h)$
 	    \STATE $t = t + 1$
      \UNTIL {\method{Convergence}}
      %
%      \FOR {predicate $r\in R$ }
%      	\STATE $\method{CrowdEvaluate}(h^{*})$
%      \ENDFOR     
          
      %\RETURN $\Phi(\hitset{Q}) = \frac{1}{|\set{Q}|} \sum_{h \in \set{Q}} l(h)$
%      \RETURN $\frac{1}{|\set{\allhits{}}|} \sum_{h \in \set{\allhits{}}} l(h)$
            \RETURN $\mathrm{Acc_t}$
%      \RETURN $\set{Q}$

	\end{algorithmic}
	\end{algorithm}

\subsection{Bringing it All Together}
\label{sec:overall_algo}

\refalg{alg:KGEalgo} presents \system{}. 
%
%We define convergence when the variance of accuracy at time $t$ $\mathrm{Acc_t}$, across multiple runs is within $\pm0.002$ range in last $10$ iterations.
% for given total budget $\budget$. 
In Lines 1-3, we build the \cgraphfull{} $G = (\allhits{} \cup \cconstraint{}, \set{E})$ and use the labels of seed set $\set{S}$ to initialize $G$.
% as well as learn certain data specific parameters of inference engine.
In lines 4-16, we repetitively run our inference mechanism, until either the 
%budget is exhausted, 
accuracy estimates have converged,
or all the \tasks{} are covered. 

\noindent
\subsubsection*{Convergence }
%($\mathrm{Acc_t}$)
In this paper, we define convergence whenever the variance of sequence of accuracy estimates [ $\mathrm{Acc_{t-k}}$,  $\hdots$,  $\mathrm{Acc_{t-1}}$, $\mathrm{Acc_t}$] is less than $\alpha$. 
We set $k=9$ and $\alpha=0.002$ for our experiments.
%
%accuracy estimate $\mathrm{Acc_t}$ to be converged when its variance lies within 
%$\pm0.002$ range in last $10$ iterations.

\noindent
In each iteration, the \task{} with the largest inferable set is identified and evaluated using crowdsourcing (Lines 5-6).  
The new inferable set $\hitset{Q}_t$ is estimated. 
These automatically annotated nodes are added to $\hitset{Q}$ (Lines 7-10). 
Finally, average of all the evaluated \task{} scores is returned as the estimated accuracy.

\section{Experiments}
\label{sec:expmnt}

To assess the effectiveness of \system{}, we ask the following questions:
\begin{itemize}
	\item How effective is \system{} in estimating KG accuracy, both at predicate-level and at overall KG-level? (\refsec{sec:main_result}). 
	\item What is the effect of coupling constraints on its performance? (\refsec{sec:effective_constraint}).
	\item And lastly, how robust is \system{} to estimating accuracy of KGs with varying quality? (\refsec{sec:ablationExp}).

\end{itemize}

%	\begin{itemize}
%		\item Compared to other competitive baselines, how effective is \system{} in estimating KG accuracy, at fine-granular level and overall level, while utilizing limited budget? (see \refsec{sec:main_result})
%		\item How robust is \system{} to noise in dataset? (see \refsec{sec:noise_eval})
%		\item Do more coupling constraints help improve \system{}'s performance? (see \refsec{sec:more_constraints})
%		\item How scalable and efficient is \system{} for large KGs? (see \refsec{subsec:scalable})
%	\end{itemize}
	
%Statistics of the dataset used, their true accuracies, and number of coupling constraints used are reported in \reftbl{tbl:datasets}. 
%All the code, datasets and MTurk evaluations used in the paper are available at 
%%\url{http://www.anonymized_url.com}.
%http://www.anonymizedUrl.com.
%%[https://github.com/malllabiisc/mall-main/blob/master/src/users/prakhar/\\multitask\_crdsrc/data.tar.gz]. \reminder{make Public repo}.
%As our focus is not to advance conventional crowdsourcing problems, like  worker quality estimation, posterior calculation of responses etc., we relax integrating complex crowdsourcing techniques with our baselines and leave it as future work.	

\subsection{Setup}
\label{sec:setup}

\begin{table}[thb]
\begin{small}
\begin{center}

%\begin{tabular}{ |c|c|c|c| }
% \hline
% Dataset & \#Facts & %$\cconstraint$ 
% Constraints & \#\task{}-set  \\
% \hline
%%	 $\mathrm{NELL_{SPORTS}}$   &  20,000 & $\cconstraint_N=$ 100 PRA paths + Dom-Ran & $\allhits{}_N=$ 1860\\ 
%%	 \hline
%%	 $\mathrm{YAGO_{SMALL}}$ & 40,940    & $\cconstraint_Y=$ 20 AMIE paths + Dom-Ran  & $\allhits{}_Y=$1000\\
%
%	 NELL sports   &  $23422$ & $|\cconstraint_N|= 130$% PRA horn clause + Domain-Range signatures 
%	 & $\allhits{}_N= 1860$\\ 
%	 \hline
%	 YAGO2 sample & $46654$    & $|\cconstraint_Y|= 28$ %AMIE horn clause + Domain-Range signatures  
%	 & $\allhits{}_Y= 1386$\\
% \hline
%\end{tabular}
% \caption{\small \label{tbl:datasets}Description of datasets used in the experiments. (See \refsec{subsec:datasets})\\}
%
% \begin{tabular}{ |c|c|c|c| }
% \hline
% \task{} set & \#Correct & \#Incorrect & Gold Acc.  \\
% \hline
%	 $\allhits{}_N$ {\small (NELL sports)}   &  $1699$ & $161$ & $91.34\%$\\ 
%	 %\hline
%	 $\allhits{}_Y$ {\small (Yago2 sample)} & $1375$ & $11$ & $99.20\%$\\ 
%	 \hline
%\end{tabular}
% \caption{\small \label{tbl:AMT} Details of \task{} subsets which were used for accuracy evaluation. Gold evaluations for these beliefs were obtained using Amazon Mechanical Turk (see \refsec{subsec:crdsrcing}).}
 
%\begin{tabular}{ |c|c|c|c|c| }
% \hline
% \task{} set & \#Beliefs & \#Preds. & Constraints & Gold Acc.  \\
% \hline
%	 $\allhits{}_N$ {\small (NELL)}  & 1860 & & $|\cconstraint_N|= 130$ & $91.34\%$\\ 
%	 %\hline
%	 $\allhits{}_Y$ {\small (Yago2)} & 1386 & & $|\cconstraint_Y|= 28$ & $99.20\%$\\ 
%	 \hline
%\end{tabular}

\begin{tabular}{ |c|c|c| }
 \hline
  Evaluation set & $\allhits{}_N$ {\small (NELL)}  & $\allhits{}_Y$ {\small (Yago2)} \\
 \hline
	\#\tasks{} & $1860$ & $1386$ \\
	\#Constraints & $|\cconstraint_N|= 130$ & $|\cconstraint_Y|= 28$ \\
	\#Predicates & $18$ & $16$ \\
	Gold Acc. & $91.34\%$ & $99.20\%$ \\
 \hline
\end{tabular}

 \caption{\small \label{tbl:AMT} Details of \task{} subsets used for accuracy evaluation. 
% Gold evaluations for these beliefs were obtained using AMT 
 (\refsec{subsec:crdsrcing}). 
% \rmd{PPT: transpose table}
 }
\end{center}
\end{small}
\end{table}

\vspace{-3mm}
\subsubsection{Setup}
\label{subsec:datasets}
%\noindent {\bf Datasets}: 
%	\begin{itemize}
%	 	\item 
%	 	{\bf NELL-sports}: 
	\subsubsection*{Datasets} We consider two automatically constructed KGs, NELL and Yago2 for experiments. From NELL\footnote{ 
	 	\scriptsize
	 	\url{http://rtw.ml.cmu.edu/rtw/resources}}, we choose a sub-graph of sports related beliefs \textbf{NELL-sports}, mostly pertaining to athletes, coaches, teams, leagues, stadiums etc. 
	 	We construct coupling constraints set $\mathbf{\cconstraint_N{}}$ using top-ranked PRA inference rules for available predicate-relations \cite{lao2011random}. 
The confidence score returned by PRA are used as weights $\theta_i$. We use NELL-ontology's predicate-signatures to get information for \textit{type} constraints. 
	 	%
	 	%As \systemfull{} is meaningful in realms of structurally related data, our set of NELL-related \tasks{} $\allhits_N{}$ had 1860 facts which were actively engaging with each other via $\cconstraint_N{}$ coupling constraints.
	 	%%%%
	 	%
	 	%
%	 	\item {\bf Yago2}: 
	 	We also select \textbf{YAGO2-sample}\footnote{
	 	\scriptsize
	 	\url{https://www.mpi-inf.mpg.de/departments/databases-and-information-systems/research/yago-naga/}
	 	%https://www.mpi-inf.mpg.de/departments/databases-and-information-systems/research/yago-naga/
	 	}  
	 	, which unlike NELL-sports, is not domain specific.
			 We use AMIE horn clauses \cite{galarraga2013amie} to construct multi-relational coupling constraints $\mathbf{\cconstraint_Y{}}$. For each $\cconstraint_i{}$, the score returned by AMIE is used as rule weight $\theta_i$.
%			 To preserve consistency in both datasets, we manually introduced predicate signatures in $\cconstraint_Y{}$ .
	% \end{itemize}
	%
	 	 \reftbl{tbl:AMT} reports the statistics of datasets used, their true accuracy and number of coupling constraints.

\subsubsection*{Size of evaluation set} In order to calculate accuracy, we require gold evaluation of all beliefs in the evaluation set. Since obtaining gold evaluation of the entire (or large subsets of) NELL and Yago2 KGs will be prohibitively expensive, we sample a subset of these KGs for evaluation. Statistics of the evaluation datasets are shown in \reftbl{tbl:AMT}.  %We work with a subset of NELL (and  Yago) belief set for evaluation %from $23,422$ to $1,860$ (and similarly Yago) because as obtaining gold labels for all KG beliefs %($23$k+$46$k) would incur huge crowdsourcing costs.
%Second, the selected $\allhits_N{}$ and $\allhits_Y{}$ subsets, instantiated from horn-clauses, actively participate in evaluation coupling and reinforce our thesis of \system{} to harness dependencies among beliefs.
%We made sure the reduced subsets are representative of all predicate-relations compared to original KGs.
We instantiated top ranked PRA and AMIE first-order-conjunctive rules over NELL-sports and Yago datasets respectively and  their participating beliefs form $\allhits_N{}$ and $\allhits_Y{}$. 
%\rmd{PPT: shown we show sample inference rules?}

\subsubsection*{Initialization} \refalg{alg:KGEalgo} requires initial seed set $\set{S}$ which 
%This initial evidence helps learn posterior for rule weights $\theta_i$.
%In the absence of $\set{S}$, 
we generate by randomly evaluating $|\set{S}|=50$ beliefs from $\allhits{}$. 
%Cost incurred in evaluating $\set{S}$ is deducted from the initial budget. 
All baselines start from $\set{S}$.
We perform all inference calculations keeping soft-truth values of \tasks{}. 
For asserting true (or false) value for beliefs, we set a high soft label confidence threshold at $\tau=0.8$ (see \refsec{subsec:eval_prop}). %i.e., $l(h)_{MAP} \geq \tau $  (we set threshold $\tau=0.8$ here) see \refsec{subsec:eval_prop}.
We do not take KG's belief confidence scores into account, as many a times KGs are confidently wrong.

\noindent
\textbf{Problem of cold start: }
\refalg{alg:KGEalgo} assumes that we are given few evaluated \tasks{} $\set{S}$ as initial seed set.
This initial evidence helps \psl{} learn posterior for rule weights and tune specific parameters.
In absence of seed set, we can generate $\set{S}$ by  randomly sampling few \tasks{} from distribution.
However, getting \tasks{} $h \in \{\set{S}\}$ crowd evaluated  also incurs cost and we must use budget judiciously to sample only enough tasks and let \system{} run thereafter. 
In all our experiments below, we have considered cold-start setting and randomly sampled $|\set{S}|=50$ tasks to train data specific parameters.

Note that random \tasks{} in $\set{S}$ are sampled from the true distribution of labels, say $\mathbb{D}$.
Running \psl{} directly over $\set{S}$ changes this distribution $\mathbb{D}$ to $\hat{\mathbb{D}}$ due sparse inferences, leading to undesirable skewing.
Hence, before calling the iterative \system{} routine, we do one-time normalization of the scores assigned by inference engine (\psl{} in our case)  to establish concordance between the accuracy estimate after initial random sampling $\set{S}$ and that of just after first iteration of \psl{}, i.e we try to make $\mathbb{D} \approx \hat{\mathbb{D}}$.  We applied class mass normalization as %\reminder{PO:verify name and give reference}
\begin{equation*}
\hat{p}(c|h) = \frac{\frac{q_c}{p_c}~p(c|h)}{\sum_{\forall c'}\frac{q_{c'}}{p_{c'}}~p(c'|h)}
\end{equation*}
where $p(c|h)$ is the probability of obtaining class $c \in \{0,1\}$ for a given \task{} $h$, $p_c$ is the current accuracy estimate by \psl{} inference and $q_c$ is estimate of initial random samples from $\set{S}$.

%\textbf{Weight learning: }

\subsubsection*{Approximate Control Mechanism}
To find the best candidate in Line 5 of \refalg{alg:KGEalgo}, inference engine runs over all the remaining (unevaluated) \tasks{}. 
Even in case of modest-sized $G$, this may not always be computationally practical.
We thus reduce this search space to a smaller set $\hitset{H'}$ and further distinguish $h^*$ by explicit calculation of $|\inferset(G,\set{Q}\cup h )|$. We observed that \tasks{} adjacent to greater number of \textit{unfulfilled} constraints in \cgraph{} propose suitable candidates for $\hitset{H'}$, where \textit{unfulfilled} constraints are  those $\cconstraint_k{}$'s which have at least one adjacent \task{} $h_k \notin \inferset(G,\set{Q} )$. 
We also experimentally validated this heuristic by varying the size to $|\hitset{H'}| \in \{3, 5, 7\}$, and observed that this caused negligible change in performance, indicating that the heuristic is indeed effective and stable. We use this strategy for all the \system{}-based experiments in this paper. Please note that \system{} without any approximation will only improve performance further, thereby making  conclusions of the paper even stronger.

	 We will release the NELL-sports and YAGO datasets used, their crowdsourced labels, coupling constraints, PRA and AMIE rules and code used for inference and control in this paper upon publication. 
	 %We hope to provide our model as template for evaluating automatically constructed KGs using minimal crowdsourcing budget. \\
%	 {\small \texttt{https://github.com/malllabiisc/kg-eval}}  
%	 \reminder{TODO} 

\subsubsection*{Acquiring Gold Accuracy and Crowdsourcing \tasks{}}

 \label{subsec:crdsrcing}

\begin{figure}[t]
	\begin{center}
		\includegraphics[scale=0.28]{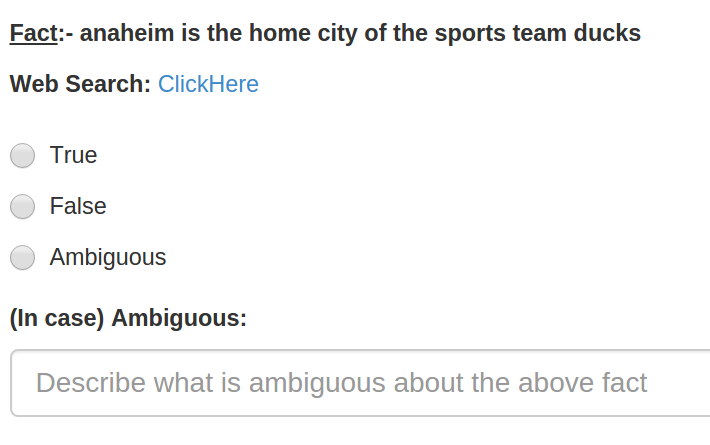}
		\caption{\small \label{fig:amt}Sample \task{} posted to Mechanical turk for evaluating  \textit{(ducks, homeCity, anaheim)} belief edge in KG. (see \refsec{subsec:crdsrcing})}
	\end{center}
	\end{figure}

% {\bf Acquiring Gold Accuracy and Crowdsourcing \tasks{}}: 
To compare \system{} predictions against human evaluations, we evaluate all \tasks{} $ \{\allhits_N{} \cup \allhits_Y{}\}$ on AMT.
%To obtain good quality responses and for the ease of workers, we used ontology to translate each triple-extraction into human readable format.
For the ease of workers, we translate each \textit{entity-relation-entity} belief into human readable format before posting to crowd. 
For instance, \textit{(Joe Louis Arena, homeStadiumOf, Red Wings)} was rendered in \tasks{} as \textit{``Stadium Joe Louis Arena is home stadium of sports team Red Wings"} and asked to label true or false. 
To further capture strange cases of machine extractions, which might not make sense to an abstracted worker, we gave the option of classifying fact as `Ambiguous' which later we  disambiguated ourselves.
Web search hyperlinks were provided to aid the workers in case they were not sure of the answer. 
\reffig{fig:amt} shows a sample \task{} posted on AMT.
%We also gave workers the option of classifying beliefs as `\textit{Ambiguous}' which we  disambiguated later.
%Web search hyperlinks were provided for worker's aid.
%\reffig{fig:amt} shows a sample \task{} posted on AMT.
%\rmd{PPT: this paragraph can be shortened}

We published \tasks{} on AMT under `classification project' category. 
AMT platform suggests high quality \textit{Master-workers} who have had a good performance record with classification tasks. 
We hired these AMT recognized master workers for high quality labels and paid \$0.01 per \task{}. 
We correlated the labels of master workers and expert labels on 100 random beliefs and compared against three aggregated non-master workers.
We observed that master labels were better correlated ($93\%$) as compared to non-masters ($89\%$),  incurring one-third the cost.
%
%
%\rmd{PPT: should we also mention the non-master comparison?}
Consequently we consider votes of master workers for $ \{\allhits_N{} \cup \allhits_Y{}\}$ as gold labels, which we would like our inference algorithm to be able to predict.
As all \tasks{} are of binary classification type, we consider uniform cost across \tasks{}.
%We published our \tasks{} under `classification project' category, employed high quality Master workers and paid \$0.01 per \task{}. 
%We maintained enough budget on AMT to let KGEval terminate after inferring all \tasks{} in \cgraph{}.
Details are presented in \reftbl{tbl:AMT}. 

Our focus, in this work, is not to address conventional problem of truth estimation from noisy crowd workers.
We resort to simple majority voting technique in our analysis of noisy workers for structurally rich \systemfull{}.
For our experiments, we consider votes of master workers for $ \{\allhits_N{} \cup \allhits_Y{}\}$ as gold labels, which we would like our inference algorithm to be able to predict. 
However, we also acknowledge several sophisticated techniques that have been proposed, like Bayesian approach, weighted majority voting etc., which are expected to perform better.

\vspace{-2mm}
\subsubsection{Performance Evaluation Metrics}
%To quantify the performance of algorithms, we measure them against the following metrics. 
%We measure the performance quality of baseline methods against the following metrics. 
Performance of various methods are evaluated using the following two metrics. 
%
%To capture accuracy at fine granular level for each predicate separately, we define  $\Delta \mathrm{Acc_{Micro}}$ as the average of difference between gold and estimated accuracy of each of the $R$ `predicate-relations' in KG.
To capture accuracy at the predicate level, we define $\predicatemetric{}$ as the average of difference between gold and estimated accuracy of each of the $R$ %`predicate-relations' 
	predicates in KG.
\begin{small}
\begin{eqnarray*}
%	\Delta \mathrm{Acc}_{\mathrm{Micro}} &  = & \frac{1}{|R|}\Bigg(\sum_{\forall r\in R}\Big|\Phi(\allhits{}_r)-\frac{1}{|\set{H{}}_r|} \sum_{\forall h \in \set{H{}}_r} l(h) \Big| \Bigg)
	\predicatemetric &  = & \frac{1}{|R|}\Bigg(\sum_{\forall r\in R}\Big|\Phi(\allhits{}_r)-\frac{1}{|\set{H{}}_r|} \sum_{\forall h \in \set{H{}}_r} l(h) \Big| \Bigg)
\end{eqnarray*}
\end{small}

We define $\overallmetric{}$ as the 
%$\Delta \mathrm{Acc_{Macro}}$ as 
difference between gold and estimated accuracy over the entire evaluation set.
\begin{small}
\begin{eqnarray*}
%	\Delta \mathrm{Acc}_{\mathrm{Macro}} & = &  \bigg|\Phi(\allhits{}) - \frac{1}{|\set{H}|}\sum_{\forall h \in \set{H}} l(h)\bigg| .
		\overallmetric  & = &  \bigg|\Phi(\allhits{}) - \frac{1}{|\set{H}|}\sum_{\forall h \in \set{H}} l(h)\bigg|
\end{eqnarray*}
\end{small}

Above, $\Phi(\allhits{})$ is the overall gold accuracy, % (by evaluating all beliefs by crowd-workers), 
$\Phi(\allhits_r{})$ is the gold accuracy of predicate $r$
% which we want our baselines to estimate (see \refsec{sec:problemform}), 
 and $l(h)$ is the label assigned by the currently evaluated method.
%$\Delta \mathrm{Acc}_{\mathrm{Macro}}$ treats entire KG as single bag of \tasks{} whereas $\Delta \mathrm{Acc}_{\mathrm{Micro}}$ segregates beliefs based on their type of predicate-relation.
%It is desirable that $\Delta\mathrm{Acc}_{\mathrm{Micro}} \to 0$ and $\Delta\mathrm{Acc}_{\mathrm{Macro}} \to 0$ for better accuracy estimation.
$\overallmetric{}$ treats entire KG as a single bag of \tasks{} whereas $\predicatemetric{}$ segregates beliefs based on their type of predicate-relation.
For both metrics, lower is better.
%It is desirable that $\Delta \mathrm{OA} \to 0$ and $\Delta \mathrm{PWA} \to 0$ for better accuracy estimation.

% in KG and computes mean average accuracy over all such types.

\subsection{Baseline Methods}
\label{sec:baselines}
%\reminder{==TODO== change the system{} name in Graph and plots}
%Sampling from structurally rich multi-relational graphs, such as KGs in our setting, is a relatively unexplored problem.
Since accuracy estimation of large multi-relational KGs is a  
%Sampling from structurally rich multi-relational KGs  is a 
relatively unexplored problem, there are no well established baselines for this task (apart from random sampling). 
We present below the baselines which we compared against \system{}.
%settings have not been well explored and  mostly rely upon heuristic selection.
%We ran experiments to compare \system{} against few such popular sampling methods as mentioned below.
%\reminder{Mention how there are not many baselines which account for structurally rich setting  like multirelational}

%\vspace{1mm}
\noindent
\textbf{Random: }
%\subsubsection*{Random baseline}
Randomly sample a \task{} $h \in \hitset{H}$ without replacement and crowdsource for its correctness. 
Selection of every subsequent \task{} is independent of previous selections.
%Results here are averaged over three such random sequences.
%In order to take away any bias introduced by particular sequence of random selection, we averaged the results over few repetitions of different random sequences and plot their error bars too.

%\vspace{1mm}
\noindent
\textbf{Max-Degree: }
%\subsubsection*{Max-Degree selection}
Sort the \tasks{} in decreasing order of their degrees in \cgraph{} and select them from top for evaluation; this method favors selection of more centrally connected \tasks{} first.
%As degree of $h_i$ depicts the cardinality of set of other \tasks{} $\{h_j\}$ which share common constraint $\cconstraint_{ij}$ in \cgraphfull{} $G$, 

Note that there is no notion of inferable set in the above two baselines. 
Individual \tasks{} are chosen and their evaluations are singularly added to compute the accuracy of KG.

%\vspace{1mm}
\noindent
\textbf{Independent Cascade: }
%\subsubsection*{Independent Cascade propagation} 
This method is based on contagion transmission model where nodes only infect their immediate neighbors  \cite{kempe2003maximizing}.
%Here, all the first order logic constraints $\cconstraint{}$ transform to mere neighborhood relation. As there is no notion of prioritization, all constraints are given equal importance.
%The \cgraphfull{} transforms to a mere neighborhood relation and there is no notion of prioritizing one constraint over the other.
%As there is no notion of prioritizing one constraint over the other, weights data of $w_i$ is lost.
%
%At every time iteration $t$, we chose $h_t \in \allhits{}$ which has the highest number of neighboring \tasks{} that are not included in inferable set $\inferset(G,\set{Q}_{t-1} \cup  h_t)$. Its crowdsourced evaluation is added to $\hitset{Q}_t$ and we let it conduct its categorization (evaluation) label to adjacent \tasks{} in \cgraph{}. This process is repeated until all the budget is exhausted, or all the nodes in \cgraph{} are covered. We can think of this baseline as a simplification of \system{} where all the relations are ignored, and the inferable set inference  is just neighborhood propagation.
At every time iteration $t$, we choose a \task{} which is not evaluated yet, crowdsource for its label  and let it propagate its evaluation label in \cgraph{}. 
%This process is repeated until the budget is exhausted or all the nodes in \cgraph{} are covered. %We can think of this baseline as a simplification of \system{} where all the relations are ignored, and the inferable set inference  is just neighborhood propagation.
%Formally, for any \task{} $h$ to be present in the final inferable set, it is necessary that $h \in \bigcup_{t=1,2,...,T}~  \inferset(\set{Q}_t)$ 
%For $h$ to be added to $\inferset(\set{Q}_t)$ at time $t$, there should exist a neighbor in \cgraph{} such that $neighbor(h)\subset \inferset(\set{Q}_{t-1})$.

%\vspace{1mm}
\noindent
\textbf{\system{}: }
%\subsubsection*{\system{}}
Method proposed in \refalg{alg:KGEalgo}. 

\subsection{Effectiveness of \system{}}
\label{sec:main_result}
%%
%\noindent
%
%\vspace{-1mm}
%\subsubsection{Predicate-wise Accuracy}
%\label{subsec:micro_result}
%%{\bf Micro Accuracy}: 

\begin{table}[t]
\begin{small}
 \centering
	\begin{tabular}{|l|c|c|c|}
  	\hline
 
	\multicolumn{4}{|c|}{NELL sports dataset ($\set{H}_N$)} \\
	\hline
	Method & $\predicatemetric$ & $\overallmetric$  & \# Queries \\
	 & $(\%)$ & $(\%)$  &  \\
		\hline
%	Random & $0.987 \pm 0.014$ & 623 \\	\hline
%	Max-Degree & $0.971 \pm 0.0$ & 1370 \\	\hline
%	Ind-Cascade & $ 0.992 \pm 0.014  $ & 232 \\	\hline
%	\system{} & {\bf $ 0.995 \pm 0.008$} & {\bf 140} \\		\hline

	Random & $4.9$ & $1.3$ &623 \\	\hline
	Max-Deg & $7.7$ & $2.9$ & 1370 \\	\hline
	Ind-Casc & $9.8$ & $0.8$ & 232 \\	\hline
	\system{} & {\bf  3.6} & {\bf 0.5} & {\bf 140} \\		\hline
	\hline
	\multicolumn{4}{|c|}{Yago dataset ($\set{H}_Y$)} \\
	\hline
	Random & $1.3$ & $0.3$ & 513 \\	\hline
	Max-Deg & $1.7$ & $0.5$  & 550 \\	\hline
	Ind-Casc & $1.1$ & $0.7$ & 649 \\	\hline
	\system{} & {\bf 0.7} & {\bf 0.1} & {\bf 204} \\		
	\hline	
	\end{tabular}
 
	\caption{\label{tbl:converge_nell_yago} $\predicatemetric (\%)$ and $\overallmetric (\%)$ estimates (lower is better) of various methods with number of crowd-evaluated queries (\task{} evaluations) to reach the $\overallmetric{}$ converged estimate. \system{} uses the least number of \task{} evaluations while achieving best estimates. This is our main result. (See \refsec{sec:main_result})}
\end{small}
\end{table}

%%
%%\begin{table}[t]
%%\begin{small}
%% \centering
%%	\begin{tabular}{|l|c|c|}
%%  	\hline
%% 
%%%	\multicolumn{3}{|c|}{Predicate level accuracy ($1-\Delta Acc_{Micro}$)} \\
%%	\multicolumn{3}{|c|}{Predicate-level accuracy ($1 -\predicatemetric{}$)} \\
%%	\hline
%%	Method & NELL-sports ($\set{H}_N$) & Yago ($\set{H}_Y$)  \\	\hline
%%
%%	Random &  $0.7516$ & $0.988$ \\	\hline
%%%	Max-Degree &  $ 0.1590$ & $0.980$ \\	\hline
%%	Ind-Cascade &  $0.9026 $ & $0.984$ \\	\hline
%%	\system{} &  {\bf 0.9641} & {\bf 0.993} \\		\hline
%%	\hline		
%%	%%\hline	
%%	\# Queries & $140$ & $204$ \\
%%	\hline
%%	
%%	\end{tabular}
%% 
%%%	\caption{\small \label{tbl:new_micro_nell_yago} Average of predicate level accuracy estimates \mbox{$1 - \predicatemetric{}$} (higher is better) of various methods.
%%%% with number of queries (\task{} evaluations) to reach this converged estimate. 
%%% \system{} uses the least number of \task{} evaluations while achieving best estimates (see \refsec{subsec:micro_result}). \rmd{PPT: make this table same as the $\overallmetric{}$ table}
%%% }
%%\end{small}
%%%\end{table}

\begin{figure}[t]
\begin{small}
	\begin{center}
		\includegraphics[width=0.40\textwidth]{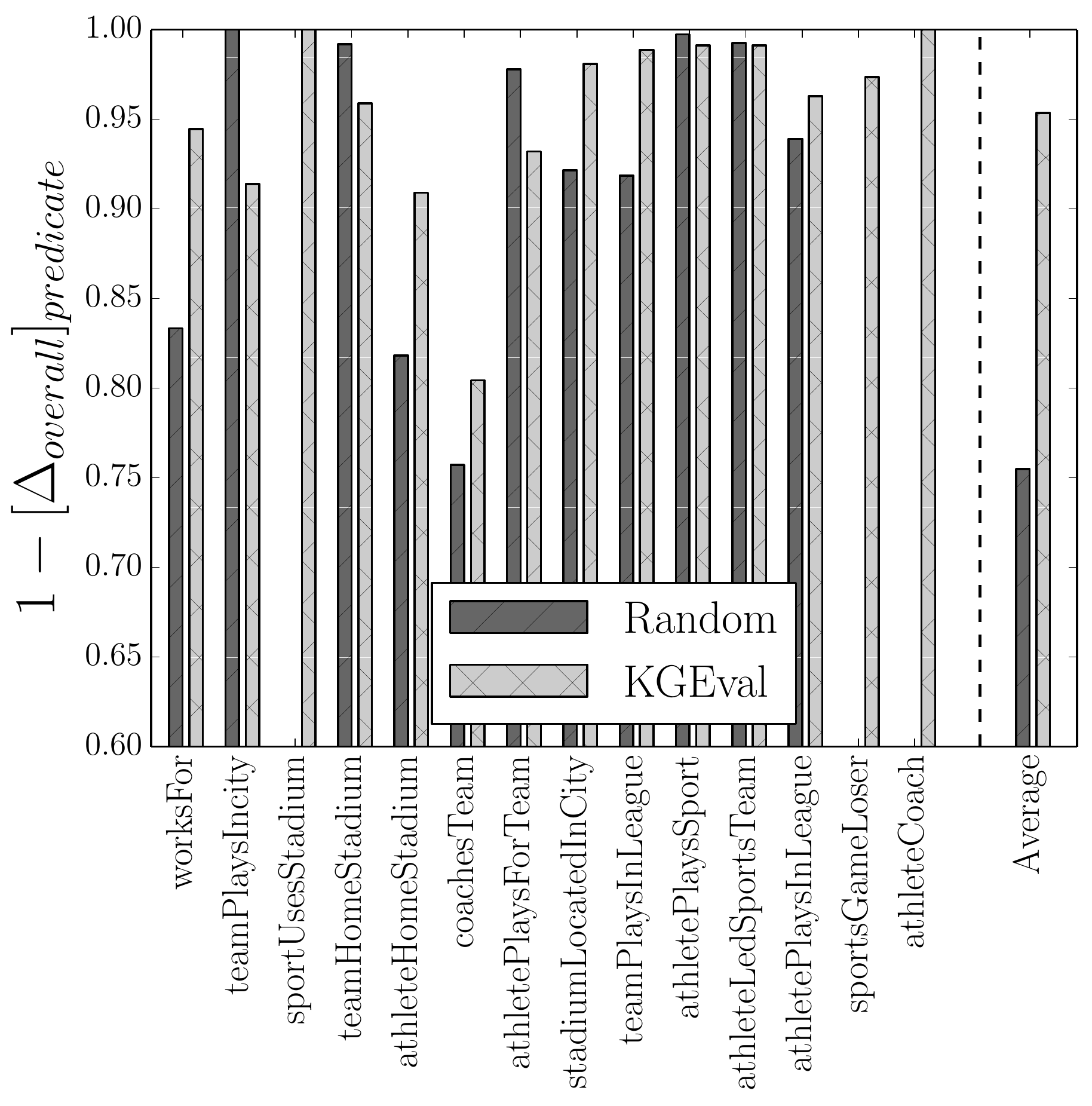}
		\caption{\small \label{fig:predicatewise}Comparing $(1 - [\overallmetric{}]_{predicate})$ of individual predicates (higher is better) in $\set{H}_N$ between \system{} and Random, the two top performing systems in \reftbl{tbl:converge_nell_yago}. $[\overallmetric{}]_{predicate}$ means $\overallmetric{}$ computed over beliefs of individual predicates. We observe that \system{} significantly outperforms Random in this task.
		(see \refsec{sec:main_result})}
	\end{center}
	\end{small}
	\end{figure}

Experimental results of all methods comparing $\overallmetric{}$ at convergence (see \refsec{sec:overall_algo} for definition of convergence), number of crowd-evaluated queries needed to reach convergence, and $\predicatemetric{}$ at that convergence point are presented in \reftbl{tbl:converge_nell_yago}. For both metrics, lower is better. From this table, we observe that \system{}, the proposed method, is able to achieve the best estimate across both datasets and metrics. This clearly demonstrates \system{}'s effectiveness. Due to the significant positive bias in $\allhits_Y{}$ (see \reftbl{tbl:AMT}), all methods do fairly well as per $\overallmetric{}$ on this dataset, even though \system{} still outperforms others.

From this table, we also observe that \system{} is able to estimate KG accuracy most accurately while utilizing least number of crowd-evaluated queries. This clearly demonstrates \system{}'s effectiveness. We note that Random is the second most effective method.

\subsubsection*{Predicate-level Analysis} For this analysis we consider the top two systems from \reftbl{tbl:converge_nell_yago}, viz., Random and \system{}, and compare performance on the $\set{H}_N$ dataset. We use $(1 - [\overallmetric{}]_{predicate})$ as the metric (higher is better), where $[\overallmetric{}]_{predicate}$ means $\overallmetric{}$ computed over beliefs of individual predicates. Here, we are interested in evaluating how well the two methods have estimated per-predicate accuracy when \system{}'s $\overallmetric{}$ has converged, i.e., after each method has already evaluated 140 queries (see $6^{th}$ row in \reftbl{tbl:converge_nell_yago}). Experimental results comparing per-predicate performances of the two methods are shown in \reffig{fig:predicatewise}. From this figure, we observe that \system{} significantly outperforms Random. \system{}'s advantage lies in its exploitation of coupling constraints among beliefs, where evaluating a belief from certain  predicate helps infer beliefs from other predicates. As Random ignores all such dependencies, it results in poor estimates even at the same level of evaluation feedback.

\subsection{Effectiveness of Coupling Constraints}
\label{sec:effective_constraint}
\label{sec:ablationExp}
%\noindent
%{\bf Effectiveness of Coupling Constraints: }

\begin{table}[t]
\begin{center}
\begin{small}
% \begin{tabular}{ |p{3.2cm}|p{2.7cm}|p{1.8cm}|p{1.8cm}|  }
 \begin{tabular}{ |l|c|c| }
 \hline
 Constraint Set & Iterations to & $\overallmetric (\%)$ \\
 & Convergence & \\
	 \hline
	 $\cconstraint{}$    &  {\bf 140}  & {\bf 0.5} \\ 
%	 \hline
%	 $\cconstraint{} - \{\set{R}_{f}\} $    &  $90$  & $0.0168$\\
%	 \hline
	 $\cconstraint{} - \cconstraint{}_{b3} $    &  $259$  & $0.9$\\
%	 \hline
	 $\cconstraint{} - \cconstraint{}_{b3} - \cconstraint{}_{b2}$    &  $335$  & $1.1$\\
%	 $\cconstraint{} - \{\set{R}_{b3}\} $    &  $209$  & $0.0083$\\
%%	 \hline
%	 $\cconstraint{} - \{\set{R}_{b2}\} - \{\set{R}_{b3}\}$    &  $285$  & $0.0106$\\
	 \hline
\end{tabular}
	\caption{\small \label{tbl:ablation}Performance of \system{} with ablated constraint sets. 
%	Observe that performance (higher is better) decreases with decreasing number of constraints, which validates the main thesis \system{} (see \refsec{sec:more_constraints}). 
Additional constraints help in better estimation with lesser iterations.(see \refsec{sec:effective_constraint})}
\end{small}
\end{center}	
\end{table}

\noindent
This paper is largely motivated by the thesis --  \textit{exploiting richer relational couplings among \tasks{} may result in faster and more accurate evaluations}. 
%To evaluate this thesis, we conducted ablation experiments over NELL-sports $\allhits_N{}$ with successively reduced coupling constraints. 
To evaluate this thesis, we successively ablated Horn clause coupling constraints of body-length 2 and 3 from $\cconstraint_N{}$. % over the NELL-sports dataset.

%\reftbl{tbl:ablation} shows the overall accuracy estimate 
%%($\Delta \mathrm{Acc_{Macro}}$) 
%($\Delta \mathrm{OA}$) 
%of \system{} and the number of iterations required until convergence.
%%$\cconstraint{}_{b2}$ and $\cconstraint{}_{b3}$ represent two sets of coupling constraints, which correspond to horn clauses of length 2 and 3, are successively ablated from $\cconstraint{}$. 
%$\cconstraint{}_{b2}$ and $\cconstraint{}_{b3}$ correspond to sets of coupling constraints with horn clauses of conjunctive body-length 2 and 3. These sets are successively removed from $\cconstraint{}$. 

We observe that with the full (non-ablated) constraint set $\cconstraint_N{}$, \system{} takes least number of crowd evaluations of \task{}s to convergence,  while providing best accuracy estimate. 
Whereas with ablated constraint sets, \system{} takes up to $2.4$x  more crowd-evaluation queries for convergence.
These results validate our thesis that exploiting relational constraints among  \tasks{} leads to effective accuracy estimation.

\subsection{Additional Experiments}
\label{sec:ablationExp}

Unless otherwise stated, we perform experiments in this section over $\set{H}_N$ using $\overallmetric$ on convergence.

%\subsection{How Effective is the Control Mechanism in isolation?}
%\subsection{Effectiveness of Control Mechanism }
%\label{sec:effec_control}
\subsubsection*{Effectiveness of \system{}'s Control Mechanism} 
We note that Random and Max-degree may be considered control-only mechanisms as they don't involve any additional inference step. In order to evaluate how these methods may perform in conjunction with a fixed inference mechanism, we replaced \system{}'s greedy control mechanism in Line 5 of \refalg{alg:KGEalgo} with these two control mechanism variants. We shall refer to the resulting systems as \textit{Random+inference} and \textit{Max-degree+inference}, respectively.

	First, we observe that both \textit{Random + inference} and \textit{Max-degree + inference} are able to estimate accuracy more accurately than their control-only variants. Secondly, even though the accuracies estimated by \textit{Random+inference} and \textit{Max-degree+inference} were comparable to that of \system{}, they required significantly larger number of crowd-evaluation queries -- $1.2$x and $1.35$x more, respectively. 
	This demonstrates \system{}'s effective greedy control mechanism.

\noindent
\subsubsection*{Rate of Coverage }
%
%\subsubsection{Rate of Coverage} %and Budget Saved}
\label{subsec:coverage_budget}
\begin{figure}[t]
	\begin{center}
%		\subfloat[]{\includegraphics[width=0.3\textwidth]{figures/nell_baselines.pdf}\label{fig:nell_main_result}}
%		\subfloat[]{\includegraphics[width=0.3\textwidth]{figures/yago_baselines.pdf}\label{fig:yago_main_result}}
		\includegraphics[width=0.3\textwidth]{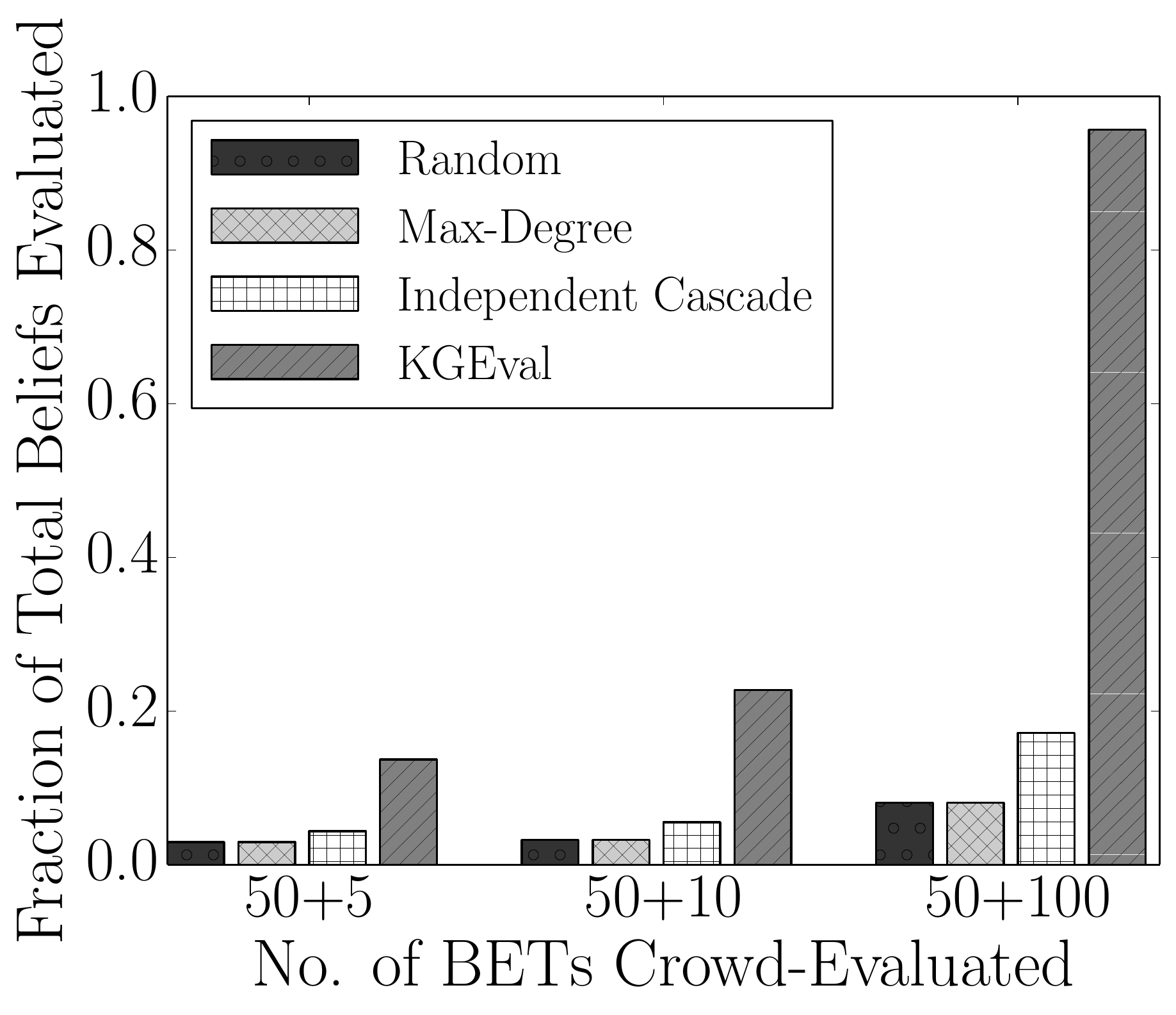}\label{fig:NELL cover}
%		\subfloat[]{\includegraphics[width=0.25\textwidth]{figures/noise/NELL_10_noVar.pdf}\label{fig:noise10}}
		
		\caption{\small  \label{fig:main_result}Fraction of  %$\allhits_N{}$ \tasks{} automatically 
		total beliefs whose evaluation where automatically inferred by different methods for varying number of crowd-evaluated queries (x-axis) in  $\set{H}_N$. 
		\system{} automatically evaluates the largest number of \tasks{} at each level (see \refsec{subsec:coverage_budget}). 
		%\rmd{PPT: X-axis: "No. of BETs Corwd-Evaluated". Y-axis should be "Fraction of Total Beliefs Automatically Evaluated"}
		} 
	\end{center}
\end{figure}
\reffig{fig:main_result} shows the fraction of total beliefs %\tasks{} covered/inferred evaluated 
whose evaluations were automatically inferred by different methods as a function of number of crowd-evaluated beliefs. We observe that \system{} infers evaluation for the largest number of \tasks{} at each supervision level. 
Such fast coverage with lower number of crowdsource queries is particularly desirable in case of  large KGs with scarce budget.
%\system{} will utilize a significantly smaller portion of the available budget compared to other methods.  

%\vspace{5mm}
%\subsubsection{Effectiveness of richer relational settings}
%\subsection{Do Additional Coupling Constraints Help?}
%\subsection{Effectiveness of Coupling Constraints}
%\label{sec:more_constraints}

%\subsection{Perturbation Experiments}
%To gain better insights of our model, we made specific alterations to our baseline algorithm and present our observations below.

%\subsection{How Robust is \system{} to Noise in KG Beliefs?}
%\subsection{Robustness to Noise}
%\label{sec:noise_eval}

\noindent
\subsubsection*{Robustness to Noise }
\noindent
%In order to examine adaptability and robustness of various algorithms in the presence of noise, we evaluated all the methods over noisy variants of the NELL dataset.
In order to test robustness of the methods in estimating accuracies of KGs with different gold accuracies, we artificially added noise to $\set{H}_N$ by flipping a fixed fraction of edges, otherwise following the same evaluation procedure as in \refsec{sec:overall_algo}.
We chose triples which were evaluated to be true by crowd workers and also had functional predicate-relation to ensure that the artificially generated beliefs were indeed false. This resulted in variants of $\set{H}_N$ with gold accuracies in the range of $81.3\%$ to $91.3\%$. 
%\rmd{PPT: fill in}. 
We analyze $\overallmetric{}$ (and not $\predicatemetric{}$) because flipping edges in KG distorts predicate-relations dependencies. 
%\rmd{PPT: not clear about the last sentence}.

We evaluated all the methods and observed that while performance of other methods degraded significantly with diminishing KG quality (more noise), \system{} was significantly robust to noise.
Across baselines \system{} best estimated $\overallmetric = 0.8\%$ with $156$ \tasks{}, whereas the next best baseline took $370$ \tasks{} to give $\overallmetric = 1.2\%$.
%
%
%\rmd{PPT: add some numbers}. 
%
This along with \system{}'s performance on Yago2, a KG with naturally high gold accuracy, suggests that \system{} is capable of successfully estimating accuracies of KGs with wide variation in quality (gold accuracy).

\subsection{Scalability and Run-time }
\label{subsec:scalable}

%\reminder{do some O() analysis}\\
%Total turn-around time, for entire \system{} framework, might not be the best metric as it involves crowd workers bottleneck.
%Eliciting crowd responses depends on several stochastic factors like \textit{time of the day}, \textit{weekday or weekend}, availability of \textit{interested workers} etc.
%However, to know the applicability of \system{} in online fashion, it is essential to know run-time for inference mechanism.
\noindent

\subsubsection*{Scalability comparisons with MLN} 
Markov Logic Networks (MLN) \cite{richardson2006MLN} is a probabilistic logic which can serve as another candidate for our Inference Mechanism. In order to compare the runtime performance of \system{} with PSL and MLN as inference engines, respective; we experimented with a subset of the NELL dataset with 1860 \tasks{} and 130 constraints. 
%Markov Logic Networks \cite{richardson2006MLN} can serve as a candidate for Inference Mechanism. 
We compared the runtime performance of \system{} with PSL and MLN as inference engines.
%, we experimented with a NELL dataset. 
While PSL took 320 seconds to complete one iteration, the MLN implementation\footnote{ \scriptsize \url{pyLMN: http://ias.cs.tum.edu/people/jain/mlns}} could not finish grounding the rules even after 7 hours. 
%Also during calculation, \psl{} retains softer truth values in continuous interval of [0,1], unlike discrete  Markov Logic Networks \cite{richardson2006MLN}.
This justifies our choice of \psl{} as the inference engine for \system{}.
%We experimented with multiple variations of NELL and Yago datasets and gauged their run-times.
%$|\allhits{}_N|$  of $1860$ binary label tasks, with $|\cconstraint{}_N|$ around 130 had approximately $25,000$ total groundings and took 320 seconds by \psl{} engine.
%However, we noticed that run-time heavily depends over total grounding which in-turn depends over graphical structure of 
%\cgraph{}. Hence we do not generalize this result to be universally applicable.
%%On the other hand, %we tested MLN based inference systems after compromising over the soft truth calculations.
%We also tested with MLN implementations which took considerably more time to even ground all the rules in graph.
%\cgraphfull{} with $(|\allhits{}_N|=1860, |\cconstraint{}=130| $ could not even be grounded in \reminder{FILL} hours. 
%This prohibitively long run-time and compromise over soft truth values made MLN not a good alternative to our inference engine.
%%\reminder{talk about MLN comparison -- dont have numbers for mln.}

%\subsubsection*{Parallelism}
\subsubsection*{Parallelism} 
Computing $\inferset(G,\set{Q} )$ for varying $\set{Q}$ involves solving independent optimization function.
The greedy step, which is also the most computationally intensive step, can easily be parallelized by distributing calculation of $\inferset(G,\set{Q} \cup h_i )$ among $i$ different computing nodes.
The final aggregator node selects $\argmax_i |\inferset(G,\set{Q}\cup h_i )|$.

\subsubsection*{Computational Optimization }
%\subsubsection*{Computational Optimization} 
\label{sec:optim}
%One complete inference cycle includes grounding of all first-order logic rules and maximizing their posterior probabilities, given evidence, as in \refeqn{eqn:MLE}.
%This computationally expensive inference step requires efficient execution for feasibility.
Grounding of all first-order logic rules and maximizing their posterior probabilities, as in \refeqn{eqn:MLE}, is computationally expensive. 
\psl{} inference engine uses Java implementation of hinge-loss Markov random fields (hl-mrf) to find the most probable explanation\cite{bach:uai13}.
It also uses relational database for efficient retrieval during rule grounding \cite{broecheler:uai10}.

\section{Related Work}
\label{sec:related}
%\vspace{-2.0mm}
Most of the previous work on evaluation of large scale KGs has resorted to random sampling, whereas crowdsourcing research has typically considered atomic allocation of tasks wherein the requester posted Human Intelligence Tasks (HITs) independently. In estimating the accuracy of knowledge bases through crowdsourcing, we find the task of knowledge corroboration \cite{kasneci2010bayesian} to be closely aligned with our motivations. 
This work proposes probabilistic model to utilize a fixed set of basic first-order logic rules for label propagation. 
However, unlike \system{}, it does not look into the budget feasibility aspect and does not try to reduce upon the number of queries to crowdsource.
%\reminder{TODO:Add few more forward citations of this paper.}

Most of the other research efforts along this line have gone into modeling individual workers and minimizing their required redundancy. 
They are mainly focused on getting a better hold on user's behavior and use it to further get better estimates of gold truth \cite{welinder2010multidimensional}.
Recent improvements use Bayesian techniques \cite{kamar2012combining,simpson2011bayesian} for predicting accuracy of classification type HITs, but they operate in much simpler atomic setting. None of them relate the outputs of HITs to one another and do not capture the relational complexities of our \systemfull{}. 

There have been models named Find-Fix-Verify which break large complex tasks, such as  editing erroneous text, into modular chunks of simpler HITs and deal with these three inter-dependent tasks \cite{bernstein2010soylent,karger2014budget}.
The kind of inter-dependency among the three micro-tasks is very specific in the sense that output of previous stage goes as input to the next stage and cost analysis, workers Allocation and performance bounds over this model are done \cite{tran2014budgetfix}.
Our model transcends this restrictive linear dependence and is more flexible/natural.
Decision theoretic approaches on constrained workflows have been employed to obtain high quality output for minimum allocated resources \cite{kolobov2013joint,bragg2013crowdsourcing}.
Crowdsourcing tasks, like collective itinerary planning \cite{zhang2012human,little2010turkit}, involves handling tasks with global constraints, but our notion of inter-dependence is again very different as compared to above model
More recent work on construction of hierarchy over domain concepts \cite{sun2015building}, top-k querying over noisy crowd data \cite{amarilli2015top}, multi-label image annotation from crowd \cite{duan2014separate,deng2014scalable}   involve crowdsourcing over dependent HITs but their goals and methods vary largely from ours.

%As our objective function formulates to maximize the size of inference set over \cgraph{}, we explore  the domains of marketing theory \cite{goldenberg2001using}, outbreak detection \cite{leskovec2007cost} and social network analysis \cite{kempe2003maximizing}, where several models like information cascade \cite{leskovec2007patterns} and diffusion of innovation \cite{rogers1995diffusion}  have been studied to gauge the influence of entities on their neighboring groups. 
Our model significantly differs from previous works in marketing theory \cite{goldenberg2001using}, outbreak detection \cite{leskovec2007cost} and social network analysis \cite{kempe2003maximizing} etc., as it operates over multi-relational modes of inference and not just singular way of connecting two entities.
%\reminder{reviewComment: Highlight this point more. Explain what information would be lost if we treated our setup as social networks.}
%
%Along the directions of budget sensitive algorithms, performance guarantees with complex aggregation methods have been given through heuristic approach for tasks which assume  equal costs. \cite{karger2014budget,azaria2012automated}.  
%Though work in \cite{tran2013efficient} relaxes over constant cost model and provides performance guarantees, it does not account for any kind of inter-relation between tasks.

Work on budget sensitive algorithm \cite{karger2014budget,azaria2012automated,tran2013efficient} provides performance guarantees over several cost models, but do not account for any inter-relation among tasks.

%\reminder{ why is our method better over active learning}
In large scale crowdsourcing, recent works have highlighted case for active learning \cite{mozafari2014scaling}. However, unlike our selection based on relational dependence at instance level tasks, active learning selection is based upon ranking generated over classifiers.

%\vspace{-1.0mm}
\section{Conclusion}
\label{sec:conclude}

\vspace{-1mm}
In this paper, we have initiated a systematic exploration into the important yet relatively unexplored problem of accuracy estimation of automatically constructed Knowledge Graphs (KGs). We have proposed \system{}, a novel method for this problem. 
%Even though large literature on KG construction exists, the problem of systematically estimating accuracy of automatically constructed KGs is largely ignored -- \system{}, based on simple theory and easy to implement algorithm, fills this gap. %which looks at the important but unexplored problem of KG evaluation under budget constraints
\system{} is an instance of a  novel crowdsourcing paradigm where dependencies among tasks presented to humans (belief evaluation in the current paper) are exploited. To the best of our knowledge, this is the first method of its kind.  
%This is in contrast to most existing crowdsourcing approaches where the tasks are assumed to be independent.
We demonstrated that the objective optimized by \system{} is in fact NP-Hard and submodular, and hence allows for the application of simple greedy algorithms with guarantees. 
Through extensive experiments on real datasets, we demonstrated effectiveness of \system{}. 
%We will make all data and code publicly available.
% against other competitive baselines. 
%\system{} is based on simple theory and is easy to implement.

%By means of this work, we hope to highlight the opportunities provided by relational structure, which is found in most real-world scenarios, and its superiority over the generic i.i.d assumption.
%%We have made all the code and data used in the paper publicly available to foster reproducible research in this area.
As part of future work, we hope to extend \system{} to incorporate varying evaluation cost, and also explore more sophisticated evaluation aggregation. 

\bibliographystyle{abbrv}			%for CIKM
\bibliography{kgeval}

%\end{small}
%\end{scriptsize}

%%%%%%%%%%%%%%%%%	APPENDIX	%%%%%%%%%%%%%%%%%

\appendix
%\begin{small}
%\begin{footnotesize}

\section{Appendix}
\label{sec:appendix}
%\label{sec:appendix_submod}
%We have tried maintaining consistency with notations in the main paper.
%

\noindent
%{\bf Submodularity: }
%A real valued function $f$, which acts over subsets of any finite set $\allhits{}$, is said to be \textit{submodular} if $\forall R,S\subset \allhits{}$ it fulfills 
%\begin{center}
%		$f(R) + f(S) \geq f(R\cup S) + f(R\cap S)$.
%\end{center}
%We call potential function $\psi$ as pairwise \textit{regular} if for all pairs of \tasks{} $\{p,q\}\in \allhits{} $ it satisfies 

\begin{proof} {\bf (for Theorem \ref{thm:submod})}
The additional utility, in terms of label inference, obtained by adding a \task{} to larger set is lesser than adding it to any smaller subset. 
By construction, any two \tasks{} which share a common factor node $\cconstraint{}_j$ are encouraged to have similar labels in $G$.

%Both, \textit{submodular} and \textit{regular} properties conform with each other and their equivalence has been studied earlier in image segmentation \cite{kolmogorov2004energy,jegelka2011submodularity}.

%
%\reminder{RESOLVE THE conflict equation num}
%
%
%
%Note that the constructed \cgraphfull{} is conducting

%By construction, the \cgraphfull{} $G = (\allhits{} \cup \cconstraint{}, \set{E})$ is \textit{conducting}, in the sense that any two \tasks{} which share a common factor node $\cconstraint{}_j$ 
%are encouraged to have similar labels. 
%reflect the idea of positive propagation of their mutual labels. 
%This notion is also reflected at implementation level by assigning positive weights $w_j$ to all $\cconstraint{}_j \in \cconstraint{}$. 
%The conducting model enforced by such $\cconstraint{}$ discourages \tasks{} with common factor node to have dissimilar labels. 
%
%
%\cite{jegelka2011submodularity} and \cite{kolmogorov2004energy} clarify the equivalence of regular energy functions and submodular set functions.
%
%
Potential functions $\psi_j$ of \refeqn{eqn:MLE} satisfy pairwise regularity property 
i.e., for all \tasks{} $\{p,q\}\in \allhits{} $ 
	\begin{equation}
		\label{eqn:regular}
		\psi(0,1) + \psi(1,0) \geq \psi(0,0) + \psi(1,1)
%				\psi_{p,q}(0,1) + \psi_{p,q}(1,0) \geq \psi_{p,q}(0,0) + \psi_{p,q}(1,1)
	\end{equation}
	where 
%	$\psi_{p,q}$ is potential function corresponding to the constraint binding \tasks{} $\{p,q\}$  and 
	$\{1,0\}$ represent true/false.
Equivalence of \textit{submodular} and \textit{regular} properties are established \cite{kolmogorov2004energy,jegelka2011submodularity}. 
%
% of \refeqn{eqn:regular}, making it submodular for \textit{a given} $\hitset{Q}$ \cite{kolmogorov2004energy}. 
Using non-negative summation property \cite{lovasz1983submodular}, $\sum_{j \in \cconstraint{} } \theta_j \psi_j$ is submodular for positive weights $\theta_j \geq 0$.

%We consider a \task{} $h$ to be confidently inferred when the soft score of its label assignment in $\inferset(G, \set{Q})$ is greater than threshold \mbox{$\tau_h \in [0,1]$}, i.e $P(l(h)|\hitset{Q}) \geq \tau_h$. 
%%Choice of random $\tau_h$ mimics the uncertainty associated in quantifying exact dependence of \tasks{} on each other.
%The label assignment made through most probable explanation of 
%\begin{eqnarray*}
%	\mathbb{P}\big(l(h)|\set{Q}\big) & = & \frac{1}{Z} \exp 
%	\Big[ -\sum\limits_{j=1}^{|\cconstraint{}'|} \theta_j \psi_j\big(h, h'\big) \Big]  \\
%	& \geq & \tau_h
%	%
%\end{eqnarray*}
%where $\cconstraint{}'$ is the reduced constraint set which is active over $h$ and $h'$ are \tasks{} sharing $\cconstraint{}'$ with $h$.

%
% ---- NOT SURE
%
We consider a \task{} $h$ to be confidently inferred when the soft score of its label assignment in $\inferset(G, \set{Q})$ is greater than threshold \mbox{$\tau_h \in [0,1]$}.
From above we know that $\mathbb{P}(l(h)|\hitset{Q})$ is submodular with respect to fixed initial set $\hitset{Q}$.
%However we are interested in global selection of $\hitset{Q}$ which maximizes the objective of \refeqn{eqn:rc_obj}.
%
Although $\max$ or $\min$ of submodular functions are not submodular in general, but \cite{kempe2003maximizing} conjectured that global function of \refeqn{eqn:rc_obj} is submodular if local threshold function $\mathbb{P}(h|\hitset{Q}) \geq \tau_h$
 respected submodularity, which holds good in our case of \refeqn{eqn:MLE}.
This conjecture was further proved in \cite{mossel2007submodularity} 
and thus making our global 
optimization function of \refeqn{eqn:rc_obj} submodular.
\end{proof}

%%==============================================================
%%==============================================================

\begin{proof} {\bf (for Theorem \ref{thm:nphard})}
%This can be proved by showing that NP-complete Set-cover Problem (SCP) is a special case of \system{}-KGE and that it 
%can be reduced to selection of $\set{Q}$ which covers ${\inferset}(G, \set{Q})$.
%From the above discussions note that for any fixed $(\set{Q}, \cconstraint{})$, \psl{} gives us a definitive way of creating inferable sets. 
%It can be shown that NP-complete Set-cover Problem (SCP) is a special case of \system{} and that it can be reduced to selection of $\set{Q}$ which covers ${\inferset}(G, \set{Q})$.
We reduce \system{} to NP-complete Set-cover Problem (SCP) so as to select $\set{Q}$ which covers ${\inferset}(G, \set{Q})$.
For the proof to remain consistent with earlier notations, we define SCP by collection of subsets $\inferset_1{}, \inferset_2{}, \hdots, \inferset_m{}$ from set $\allhits{} = \{ h_1,h_2,\hdots, h_n\}$ and we want to determine if there exist $k$ subsets whose union equals $\allhits{}$. 
%For any fixed $(\set{Q}, \cconstraint{})$, \psl{} gives a definitive way of creating inferable sets. 
%Without loss of generality assume that optimal solution has $k \neq m $ and for all practical purposes $ k<n<m$. 
We define a bipartite graph with $m+n$ nodes corresponding to $\inferset_i{}$'s and $h_j$'s respectively and construct edge $(\inferset_i{},h_j)$ if $h_j \in \inferset_i{}$. 
%SCP can help us decide if there is a set $\set{Q}$, with cardinality k, such that $|\inferset(G, \set{Q})| \geq n+k$. 
We need to find a set $\set{Q}$, with cardinality k, such that $|\inferset(G, \set{Q})| \geq n+k$. 

Choosing our \task{}-set $\set{Q}$ from SCP solution and further inferring evaluations of other remaining \tasks{} using \psl{} will solve the problem in hand.
\end{proof}

%%==============================================================
%%==============================================================

%%%%%%%%%%%%%%

\subsection{Noisy Crowd Workers and Budget}
\label{sec:budget_alloc}
%\reminder{explain the limitations of AMT that we cant really choose the workers and workers choose the task they want to do.}
%In the discussion so far, we have not addressed the issue that the labels provided by crowd workers may be noisy. Hence, one has to redundantly post the same task to multiple workers and estimate its accuracy by aggregating all responses. In this section, we provide upper bounds on the number of workers which should evaluate a given \task{}.
%We have not addressed the issues of noisy crowd workers and limited budget so far. 
%One has to redundantly post the same \task{} to multiple workers which causes rapid depletion of budget.
% and estimate accuracy by aggregating responses.
%We resort to simple majority voting in our analysis as our focus here is not to address conventional problem of truth estimation from noisy responses.
%As the focus of this work is not to address the conventional problem of truth estimation from noisy responses, 
%
%However, we acknowledge several other sophisticated techniques like Bayesian truth estimation, weighted majority voting etc., which are expected to perform better.
Here, we provide a scheme to allot crowd workers so as to remain within specified budget and upper bound total error on accuracy estimate.
We have not integrated this mechanism with \refalg{alg:KGEalgo} to maintain its simplicity.
 
%As we do not focus on conventional problem of truth estimation from noisy responses, we resort to majority voting in our analysis. 
%However, we acknowledge other sophisticated techniques like Bayesian truth estimation, weighted majority voting etc., which are expected to perform better.
%This assumption, together with the widely accepted notion of `\textit{wisdom of crowd}', leads us to a stronger conclusion: 
%We assume that noise in crowd responses are genuine mistakes due to the lack of expertise and not due to adversarial workers.
We resort to majority voting in our analysis and assume that crowd workers are not adversarial.
So expectation over responses $r_h(u)$ for a task $h$ with respect to multiple workers $u$ is close to its true label $t(h)$ \cite{tran2013efficient}, i.e.,

\begin{equation}
\label{eqn:error_expectation}
\bigg|\mathbb{E}_{u\sim \mathbb{D}(u,h)}[r_h(u)] - t{(h)}\bigg| \leq \frac{1}{2} 
\end{equation}
where $\mathbb{D}$ is  joint probability distribution of workers $u$ and tasks $h$.
%We employ majority voting for response aggregation. % and deduce the answer $\hat{r}_h$ for task $h$.
% having $n_h$ responses 
%\begin{equation*}
%\hat{r}_h = \Big\lfloor \frac{1}{n_h} \sum_{i=1}^{n_h} r_h(u_i) - \frac{1}{2} \Big\rfloor + 1
%\end{equation*}  
%
%We also acknowledge more rich methods like weighted majority voting, bayesian expectation, graph based models etc. \reminder{CITE} but restrict ourselves to majority voting.
%For fixed budget we need to know the number of tasks at hand

%Our key idea is that is that, we want to be more confident about labels of those beliefs which directly impact larger subset of remaining \tasks{}. 
%In other words, we want to be more certain about those \tasks{} $h$ which have larger inferable set $|\inferset(G,\set{Q}\cup \{h\})|$ and hence allocate them more budget to post to more workers. 
Our key idea is that we want to be more confident about \tasks{} $h$ with larger inferable set  (as they impact larger parts of KG) and hence allocate them more budget to post to more workers. 
%
%The rationale is to be more confident about tasks which directly impact larger subset of remaining \tasks{}.
%Earlier works have considered allocation of workers to varying-cost model  \cite{tran2013efficient}, 
%but are not directly applicable in this setting due to our preference bias for few tasks over others.
%We differ from previous works on varying-cost models \cite{tran2013efficient} due to our preference bias for few tasks over others.
%
%We divide the entire budget $B$ among $n$ tasks $\{h_1, \hdots h_n\}$ and 
We determine the number  of workers $\{w_{h_1},\hdots,w_{h_n}\}$ for each task such that $h_t$ with larger inference set have higher $w_{h_t}$. 
%
%Suppose, the cardinality of set $\inferset(G,\set{Q} \cup  h_t)$  is given by $i_t$ and the cost of querying each task is $c$, then we allocate 
%%\begin{equation}
%$w_{h_t} = \big\lfloor \frac{B~ i_t~(1-\gamma)}{ c~ i_{max}} \big\rfloor $
%%\end{equation}
%where $i_{max}$ is the size of largest inferable set and $\gamma$ is constant.
%
%
For total budget $B$, 
we allocate 
%\begin{equation*}
\begin{center}
%\mbox{$w_{h_t} = \big\lfloor \frac{B~ i_t~(1-\gamma)}{ c~ i_{max}} \big\rfloor $}
$w_{h_t} = \big\lfloor \frac{B~ i_t~(1-\gamma)}{ c~ i_{max}} \big\rfloor $
\end{center}
%\end{equation*}
where $i_t$ denotes the cardinality of inferable set $\inferset(G,\set{Q} \cup  h_t)$, $c$ the cost of querying crowd worker, $i_{max}$ the size of largest inferable set and $\gamma \in [0,1]$ constant.

This allocation mechanism easily integrates with \refalg{alg:KGEalgo}; in (Line 8) we determine size of inferable set $i_t =|\set{Q}_t|$ for task $h$  and allocate $w_{h}$ crowd workers. Budget depletion (Line 9) is modified to $B_r = B_r - w_{h}c(h)$.
The following theorem bounds the error with such allocation scheme. 
%Detailed proof is in Appendix. %\refsec{sec:appendix}.
%\begin{frm-thm}
\begin{theorem}
\label{thm:error_bound}
{\bf [Error bounds]} The allocation scheme 
%in \refsec{sec:budget_alloc} 
of redundantly posing $h_t$ to $w_{h_t}$ workers does not exceed the total budget $B$ and its expected estimation error is upper bounded by $e^{-O(i_t)}$, keeping other parameters fixed. 
The expected estimation error over all tasks is upper bounded by $e^{-O(B)}$.
\end{theorem}
%\end{frm-thm}

\begin{proof} 
%{\bf (for Theorem \ref{thm:error_bound})}
Let $\gamma \in [0,1]$  control the reduction in size of inferable set by \mbox{ $i_{t+1} = \gamma~i_t$}.
%, with $\gamma= 1 - \frac{i_{max}}{\mathrm{average}_t(i_t)}$. 
By allocating $w_{h_t}$ redundant workers for task $h_t,\forall t\in \{1\cdots n\}$  with size of inferable set $i_t$, we incur total cost of
\begin{eqnarray*}
\sum_{t=1}^n c~w_{h_t}  & = & \sum_{t=1}^n \frac{B ~ i_t ~(1-\gamma)}{c ~ i_{max}} \cdot c \\
& = & \left( \sum_{t=1}^n  i_t \right) \cdot \left( \frac{B~(1-\gamma)}{i_{max}} \right) \\
\end{eqnarray*}

For greedy control mechanism, $i_1=i_{max}$. Suppose we approximate the reduction in size of inferable set by a diminishing factor of $\gamma \in (0,1)$ i.e $i_{t+1} = \gamma~i_t$.
\begin{eqnarray*}
%\sum_{t=1}^T c~n_{h_t}  & = & \sum_{t=1}^T \frac{B ~ i_t ~(1-\gamma)}{c ~ i_{max}} \cdot c \\
& = & \left( \frac{i_{max}~(1-\gamma^T)}{(1-\gamma)} \right) \cdot \left( \frac{B~(1-\gamma)}{i_{max}} \right) \\
& \leq & B
\end{eqnarray*}
Note that the above geometric approximation, which practically holds true when decaying is averaged over few time steps, helps in getting an estimate of $\sum_{t=1}^T  i_t$ at iteration $t \leq T$. 
Such approximation would not be possible unless we already ran the entire experiment.
For practical implementation, we can use an educated guess of $\gamma= 1 - \frac{i_{max}}{\mathrm{avg}_t(i_t)}$.

%& = & \left( \frac{i_{max}~(1-\gamma^T)}{(1-\gamma)} \right) \cdot \left( \frac{B~(1-\gamma)}{i_{max}} \right) \\
%
%& \leq & B
%%& \leq & B
%\end{eqnarray*}
%
%$i_1=i_{max}$ for greedy control mechanism. 
%We denote the reduction in size of inferable set by $\gamma \in [0,1]$ i.e. \mbox{ $i_{t+1} = \gamma~i_t$}, with $\gamma= 1 - \frac{i_{max}}{\mathrm{average}_t(i_t)}$.
%\begin{eqnarray*}
%\sum_{t=1}^T c~n_{h_t}  & = & \sum_{t=1}^T \frac{B ~ i_t ~(1-\gamma)}{c ~ i_{max}} \cdot c \\
%& = & \left( \frac{i_{max}~(1-\gamma^T)}{(1-\gamma)} \right) \cdot \left( \frac{B~(1-\gamma)}{i_{max}} \right) \\
%& \leq & B
%\end{eqnarray*}
%
%Note that the above geometric approximation, which practically holds true when decay in size is averaged over few time steps, helps in getting an estimate of summation $\sum_{t=1}^n  i_t$ at iteration $t \leq n$. 
%Such estimation would not be possible unless we run the entire experiment.

\noindent
{\bf Error Bounds: } 
Here we show that the expected error of estimating of $h_t$ for any time $t$ decreases exponentially in the size of inferable set $i_t$. 
%Let the worker responses for  $h_t$ binary classification task, denoted by $r_{h_t} \in \{0,1\}$, be aggregated using majority voting over users $u_k$ by
We use majority voting to aggregate $w_{h_t}$ worker responses for $h_t$, denoted by $\hat{r}_{h_t} \in \{0,1\}$
	\begin{equation}
	\hat{r}_{h_t} = \Bigg\lfloor \frac{1}{w_{h_t}} \sum_{k=1}^{w_{h_t}} r_{h_t}(u_k) - \frac{1}{2}  \Bigg\rfloor + 1
	\end{equation}
where $r_{h_t}(u_k)$ is the response by $k^{th}$ worker for $h_t$.
%
%For given task $h_t$, these $w_{h_t}$ responses are i.i.d samples from distribution of users and response $\mathbb{P}(h_t,u)$.
%Let the true label for task $h_t$ be denoted by $t(h_t)$. 
The error from aggregated response can be  given by 
$\Delta({h_t}) = | \hat{r}_{h_t} - t(h_t) |$, where $t(h_t)$ is its true label. 
From \refeqn{eqn:error_expectation} and Hoeffding-Azuma bounds over $w_{h_t}$ i.i.d responses and error margin $\varepsilon_t $, we have 
\begin{eqnarray*}
\label{eqn:hoeffding}
\Delta({h_t}) & = & \mathbb{P}\Bigg\lbrace\Bigg| \frac{1}{w_{h_t}} \sum_{k=1}^{w_{h_t}} r_{h_t}(u_k) - \mathbb{E}(r_{h}(u))\Bigg| \geq \varepsilon_t  \Bigg\rbrace \\
& \leq & 2e^{-2w_{h_t}\varepsilon_t^2} \\
& = & 2~ \mathrm{exp} \left(-2~\frac{B ~ i_t ~(1-\gamma)}{c ~ i_{max}}~\varepsilon_t^2 \right)
\end{eqnarray*}
For fixed budget $B$ and given error margin $\varepsilon_t$, we have $\Delta({h_t})=e^{-O(i_t)}$.
%
%Summing up over all tasks $t$, we get the total expected error from absolute ground truth \mbox { $\Delta(B) =  \sum_{t=1}^T \Delta({q_t})$}. 
%
%From our model (see \refsec{sec:probform}), we know that the aggregated label of task $h_t$ are  correlated with labels of its inferable set.
%As the labels are correlated the probability of making error in  $h \in \inferset(G,\set{Q}\cup h_t )$  is same as making error in $h_t \in \allhits{}$.
%, given by
%% as per out first-order-logic constraint model \ref{sec:problemform} \reminder{ check Section}. 
%\begin{eqnarray*}
%\Delta\bigg(\inferset(G,\set{Q}\cup h_t )\bigg) & \leq &  2~ \mathrm{exp} \left(-2~\frac{B ~ i_{min} ~(1-\gamma)}{c ~ i_{max}}~\varepsilon_{min}^2 \right)
%%\Delta\bigg(\inferset(G,\set{Q}\cup h_t )\bigg) & \leq &  2~ \mathrm{exp} \left(-2~\frac{B ~ i_t ~(1-\gamma)}{c ~ i_{max}}~\varepsilon_t^2 \right) \\
%%& \leq & 2~ \mathrm{exp} \left(-2~\frac{B ~ i_{min} ~(1-\gamma)}{c ~ i_{max}}~\varepsilon_{min}^2 \right)
%\end{eqnarray*}
%
Summing up over all tasks $t$, by union bounds we get the total expected error from absolute truth as $\Delta(B) =  \sum_{t=1}^n \Delta({h_t})$.
\begin{eqnarray*}
\Delta(B) & \leq & \sum_{t=1}^n  2~ \mathrm{exp} \left(-2~\frac{B ~ i_t ~(1-\gamma)}{c ~ i_{max}}~\varepsilon_t^2 \right) \\
& \leq & n \cdot 2~ \mathrm{exp} \left(-2~\frac{B ~ i_{min} ~(1-\gamma)}{c ~ i_{max}}~\varepsilon_{min}^2 \right)
\end{eqnarray*}
The accuracy estimation error will decay exponentially with increase in total budget for fixed parameters.
\end{proof}
% \reminder{PARAPHRASE and write properly}.

%\end{footnotesize}
%\end{small}

%\balance		%%VLDB

% That's all folks!
\end{document}